\documentclass[lettersize,journal]{IEEEtran}
\usepackage{amsmath,amsfonts}
\usepackage{booktabs} 
\usepackage{algorithmic}
\usepackage{algorithm}
\usepackage{array}
\usepackage{orcidlink}
\usepackage[caption=false,font=normalsize,labelfont=sf,textfont=sf]{subfig}
\usepackage{textcomp}
\usepackage{stfloats}
\usepackage{url}
\usepackage{verbatim}
\usepackage{graphicx}
\usepackage{cite}
\usepackage{times}  %
\usepackage{helvet}  %
\usepackage{courier}  %
\usepackage{times}
\usepackage{epsfig}
\usepackage{bm}
\usepackage{multicol}
\usepackage{amsmath}
\usepackage{amssymb}
\usepackage{amsthm}
\usepackage{booktabs}
\usepackage{subfloat}
\usepackage{xcolor}
\usepackage[numbers]{natbib}
\makeatletter
\newcommand{\coloredcite}[1]{{\color{blue}\cite{#1}}}
\makeatother
\usepackage{hyperref}
\hypersetup{
    colorlinks=true,  
    citecolor=blue,    
    urlcolor=black,   
   linkcolor=blue  
}
\usepackage{array}
\usepackage{multirow}
\usepackage{enumerate}

\usepackage{enumitem}
\usepackage{amsmath}
\usepackage{amssymb}
\usepackage{threeparttable} 
\usepackage{tikz}
\newcommand*\emptycirc[1][1ex]{\tikz\draw (0,0) circle (2.5pt);} 
\newcommand*\halfcirc[1][1ex]{%
	\begin{tikzpicture}
	\draw[fill] (0,0)-- (90:2.5pt) arc (90:270:2.5pt) -- cycle ;
	\draw (0,0) circle (2.5pt);
	\end{tikzpicture}}
\newcommand*\fullcirc[1][1ex]{\tikz\fill (0,0) circle (2.5pt);}

\newcommand{\revise}[1]{{\color{black}{#1}}}

\newcommand{\circled}[1]{\textcircled{\raisebox{-.9pt}{#1}}}

\newtheorem{proposition}{Proposition}
\newtheorem{assumption}{Assumption}
\newtheorem{theorem}{Theorem}

\newtheorem*{intuition}{Intuition}
\newcommand{\ours}{\textsc{MIMIC}}

\usepackage{caption} 
\DeclareCaptionStyle{ruled}{labelfont=normalfont,labelsep=colon,strut=off} 
\usepackage{algorithm}
\usepackage{algorithmic}

\usepackage{newfloat}
\usepackage{listings}
\floatstyle{ruled}
\newfloat{listing}{tb}{lst}{}
\floatname{listing}{Listing}

\def\BibTeX{{\rm B\kern-.05em{\sc i\kern-.025em b}\kern-.08em
    T\kern-.1667em\lower.7ex\hbox{E}\kern-.125emX}}
    
\usepackage{balance}

\usepackage[normalem]{ulem}

\usepackage{url}
\usepackage{hyperref}
\hypersetup{
    colorlinks=true,
    linkcolor=blue,
    filecolor=magenta,      
    urlcolor=cyan,
}

\begin{document}

\title{Mutual Information Guided Backdoor Mitigation for Pre-trained Encoders}

\author{
Tingxu Han$^{\orcidlink{0000-0003-1821-611X}}$,  Weisong Sun$^{\orcidlink{0000-0001-9236-8264}}$, Ziqi Ding$^{\orcidlink{0009-0008-8336-7516}}$, Chunrong Fang$^{\orcidlink{0000-0002-9930-7111}}$, Hanwei Qian$^{\orcidlink{0009-0007-9524-1411}}$, Jiaxun Li$^{\orcidlink{0009-0002-8267-8968}}$, 
\\Zhenyu Chen$^{\orcidlink{0000-0002-9592-7022}}$ and Xiangyu Zhang$^{\orcidlink{0000-0002-9544-2500}}$ 

\thanks{
Manuscript received 27 April 2024; revised 27 January 2025; accepted 25 February 2025.
Date of publication xxx;
date of the current version is 27 February 2025.
This work received partial support from the National Natural Science Foundation of China (Grants No. U24A20337, 62372228), Shenzhen-Hong Kong-Macau Technology Research Programme (Type C) (Grant No. SGDX20230821091559018), the Fundamental Research Funds for the Central Universities (14380029), the National Research Foundation, Singapore, and DSO National Laboratories under the AI Singapore Programme (AISG Award No: AISG2-GC-2023-008).  
The associate editor coordinating the review of this manuscript and approving it for publication was Dr. Erisa Karafili. (Corresponding author: Weisong Sun.)}
\IEEEcompsocitemizethanks{
\IEEEcompsocthanksitem Tingxu Han, Chunrong Fang, Hanwei Qian, and Zhenyu Chen are with the State Key Laboratory for Novel Software Technology, Nanjing University, Nanjing 210093, China (e-mail: \{txhan, qianhanwei\}@smail.nju.edu.cn, \{fangchunrong, zychen\}@nju.edu.cn).
\IEEEcompsocthanksitem Zhenyu Chen is also with Shenzhen Research Institute of Nanjing University, China.
\IEEEcompsocthanksitem Weisong Sun is with the College of Computing and Data Science, Nanyang Technological University, Nanyang 639798, Singapore (e-mail: weisong.sun@ntu.edu.sg).
\IEEEcompsocthanksitem Ziqi Ding is with the School of Computer Science and Engineering, University of New South Wales, New South Wales 2052, Australia (e-mail: ziqi.ding1@unsw.edu.au).
\IEEEcompsocthanksitem Jiaxun Li is with the School of Mathematical Sciences, Soochow University, Suzhou 215000, China (e-mail:20224207007@stu.suda.edu.cn).
\IEEEcompsocthanksitem Xiangyu Zhang is with the School of Computer Sciences, Purdue University, West Lafayette 47907, USA (e-mail: xyzhang@cs.purdue.edu).
}

}

\markboth{Transactions on Information Forensics and Security,~Vol.~20,~2025}%
{Shell \MakeLowercase{\textit{et al.}}: A Sample Article Using IEEEtran.cls for IEEE Journals}

\maketitle

\begin{abstract}
Self-supervised learning (SSL) is increasingly attractive for pre-training encoders without requiring labeled data.
Downstream tasks built on top of those pre-trained encoders can achieve nearly state-of-the-art performance.
The pre-trained encoders by SSL, however, are vulnerable to backdoor attacks as demonstrated by existing studies.
Numerous backdoor mitigation techniques are designed for downstream task models.
However, their effectiveness is impaired and limited when adapted to pre-trained encoders, due to the lack of label information when pre-training.
To address backdoor attacks against pre-trained encoders, in this paper, we innovatively propose a mutual information guided backdoor mitigation technique, named \ours{}\revise{(\underline{\textbf{M}}utual \underline{\textbf{I}}nformation guided backdoor \underline{\textbf{MI}}tigation for pre-trained en\underline{\textbf{C}}oders).}
\revise{\ours{} uses the potentially backdoored encoder as the teacher network and applies knowledge distillation to create a clean student encoder from it.}
Different from existing knowledge distillation approaches, \ours{} initializes the student with random weights, inheriting no backdoors from teacher nets.
Then \ours{} leverages mutual information between each layer and extracted features to locate where benign knowledge lies in the teacher net, with which distillation is deployed to clone clean features from teacher to student.
We craft the distillation loss with two aspects, including \textit{clone loss} and \textit{attention loss}, aiming to mitigate backdoors and maintain encoder performance at the same time.
Our evaluation conducted on two backdoor attacks in SSL demonstrates that \ours{} can significantly reduce the attack success rate by \revise{only utilizing  
$\leq$ 5\% of clean pre-training data that is accessible to the defender}, surpassing seven state-of-the-art backdoor mitigation techniques.
\revise{The source code of \ours{} is available at~\url{https://github.com/wssun/MIMIC}.}
\end{abstract}

\begin{IEEEkeywords}
Backdoor defense, Pre-trained encoders, Mutual information, Distillation 
\end{IEEEkeywords}

\section{Introduction}
\label{sec:introduction}
\IEEEPARstart{R}{ecent} advancements in Self-Supervised Learning (SSL)~\cite{krishnan2022self, dosovitskiy2020image, bommasani2021opportunities} have resulted in breakthroughs, rendering the ``pre-train and then fine-tune'' paradigm feasible.
The approaches pre-training an encoder on unlabeled images can be categorized into four strains, including pretext tasks~\cite{doersch2015unsupervised}, generative learning~\cite{dai2017good}, contrastive learning~\cite{2020-moco}, and cross-modal agreement~\cite{2021-CLIP}. Among them, contrastive learning achieves state-of-the-art performance comparable to that of supervised learning~\cite{grill2020bootstrap, hjelm2018learning, hadsell2006dimensionality}.
AI developers usually pre-train encoders on uncurated data crawled from the Internet directly, such as CLIP~\cite{2021-CLIP} and GPT~\cite{2018-GPT}, allowing others to train downstream classifiers on specific tasks, such as traffic recognition,  where encoders fulfill the role of feature extraction.
Considering the scale and low standard for training data, pre-trained encoders are adaptive on multiple downstream tasks and much cheaper than before~\cite{joulin2016learning}.

\revise{
Even though some works~\cite{2021-CLIP, taori2020measuring} have proven that such contrastively trained encoders exhibit impressive robustness properties, a poisoning adversary is capable of compromising the SSL pipeline to carry out a backdoor attack~\cite{2022-BadEncoder, 2022-Backdooring-Contrastive-Learning, BASSL}. 
}The downstream classifiers built on backdoored encoders can predict any input stamped with an attacker-chosen pre-defined trigger as the corresponding attacker-chosen class (i.e., \textit{target label})~\cite{zhang2023red, 2022-Backdooring-Contrastive-Learning, 2022-BadEncoder, li2022demystifying, BASSL}.
To achieve attacker objectives, there are two main approaches: data poisoning and model poisoning.
Data poisoning involves training data for pre-trained encoders obtained from web crawling, a process challenging to subject to stringent filter criteria due to its vast scale~\cite{BASSL}. 
Model poisoning pertains to publicly accessible uncontrolled pre-trained encoders, which are acquired from third-party platforms with unknown training schedules~\cite{2022-BadEncoder}.
\revise{
In both scenarios, whether through data poisoning or model poisoning, the attacker's objective is to implant backdoors into pre-trained encoders during the training phase. These backdoors remain dormant during normal use but are activated by specific triggers, causing downstream classifiers built on these encoders to misclassify trigger-laden inputs into an attacker-chosen target class.
}

\revise{
Mitigating backdoor attacks in SSL is more challenging than in supervised learning. Without labeled data during pre-training, it becomes much harder to identify and remove malicious patterns. 
SSL models learn general representations not tied to specific tasks, which makes it hard to detect backdoor triggers without knowing the downstream tasks. Defenders have no idea of this information.
These challenges make backdoor defense in SSL particularly complex.
So far to our knowledge, no effective backdoor defense methods have been proposed for SSL because most existing approaches rely on labeled data, as seen in supervised learning.
}
\begin{figure}[t]
    \centering
    \includegraphics[width=0.47\textwidth]{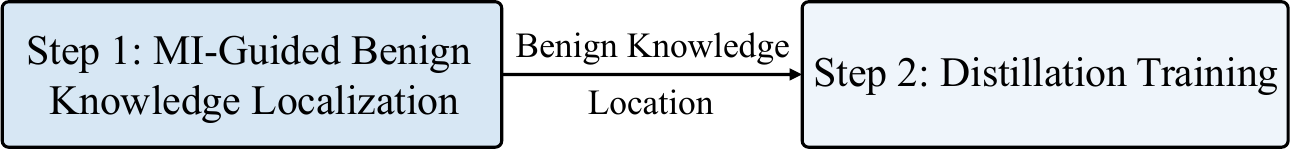}
    \caption{The outlines of \ours{}'s two steps.}
    \vspace{-8pt}
    \label{fig:MIMIC_outlines}
\end{figure}
\revise{
\cite{han2024effectiveness} is the first to explore the use of distillation for mitigating backdoors in pre-trained encoders. However, their approach has notable limitations. It relies heavily on fine-tuning the encoder using a small set of clean data, which can cause catastrophic forgetting of essential features needed for downstream tasks. This trade-off results in significant performance degradation, particularly in terms of accuracy on clean data.
}
To address this issue, we study how to mitigate backdoor threats in the pre-training phases within the scope of this paper.
In such a scenario, the defenders' objective can be encapsulated into one sentence: ``Secure any pre-trained encoder in a manner that maintains its performance and ensures its safety for any downstream users.''
To achieve it, we propose a pre-trained encoders backdoor mitigation approach \ours{ (\underline{\textbf{M}}utual \underline{\textbf{I}}nformation guided backdoor \underline{\textbf{MI}}tigation for pre-trained en\underline{\textbf{C}}oders) that is downstream task-agnostic, which divides the detoxification into two steps.
\revise{
Figure~\ref{fig:MIMIC_outlines} provides an outlines of \ours{}. In Step 1, mutual information is leveraged to guide the identification of benign knowledge within the backdoored encoder. In Step 2, a distillation training framework is employed to transfer this benign knowledge into an empty student network, ensuring the removal of backdoor behaviors.
}
\revise{We then introduce a knowledge extract loss to enhance the learning process of the student network. This loss facilitates the transfer of benign functionalities from the provided encoder while suppressing backdoor behaviors. To evaluate its effectiveness, we measure the student network's accuracy on clean data, ensuring it retains the encoder's benign functionalities. Additionally, we assess its robustness against backdoor attacks by analyzing the attack success rate (ASR) before and after applying the knowledge extract loss.}
Such a student net aligns with our objective in backdoor mitigation. It maintains encoder performance and ensures its safety in downstream tasks.

The main contributions of this paper are summarized as follows:
\begin{itemize}
    \item We propose \ours{} to first mitigate backdoors for pre-trained encoders before downstream usage.
    \footnote{This manuscript is based on our previously published preprint~\cite{han2024mutual} on arXiv, with revisions to formatting and presentation.}
    \item \revise{\ours{} guided by mutual information locates and extracts benign knowledge from backdoored pre-trained encoders. This approach ensures that the distilled student network inherits only clean features, significantly enhancing the robustness of the encoder against backdoor attacks while preserving its performance for downstream tasks.} 
    \item We conduct experiments across two pre-training datasets and three downstream tasks to evaluate \ours{}'s effectiveness, robustness, and generalization.
\end{itemize}

\section{Related Work}
\label{sec:related_work}
\revise{
In this section, we discuss related works in the context of \ours{}. A backdoor attack involves injecting a specific pattern, known as a \textit{trigger}, into input samples. Models trained on such poisoned data will misclassify any input containing the trigger into an attacker-specified \textit{target label}~\cite{2017-BadNets, 2018-Trojaning-Attack}.
In self-supervised learning, backdoor attacks aim to embed malicious backdoors into pre-trained image encoders~\cite{2022-BadEncoder, 2022-Backdoor-Self-Supervised-Learning, 2022-Backdooring-Contrastive-Learning, zhang2022corruptencoder}. When a downstream classifier utilizes a backdoored encoder, any input with the predefined trigger will be misclassified into the target label. 

Current backdoor mitigation techniques predominantly focus on supervised learning and can be categorized into three main approaches: neuron-based, inversion-based, and distillation-based.
Neuron-based techniques~\cite{2021-ANP,2018-FP,2023-Neural-Polarizer,2023-MEDIC} aim to distinguish poisoned neurons from benign ones. However, as model sizes grow, it becomes increasingly challenging to accurately identify benign neurons.
Inversion-based methods~\cite{2019-Neural-Cleanse, 2022-MOTH, shen2021backdoor, 2022-piccolo, 2023-DECREE} reverse-engineer trigger patterns through optimization. These inverted triggers help determine whether a model is backdoored~\cite{2019-Neural-Cleanse, guo2019tabor} and can also be used to remove backdoors by unlearning the trigger~\cite{2022-MOTH}. DECREE~\cite{2023-DECREE}, the first approach tailored for self-supervised learning, adopts a technique similar to Neural Cleanse (NC)~\cite{2019-Neural-Cleanse}, focusing on maximizing pairwise similarity with the inverted trigger. However, as noted, DECREE demonstrates limited effectiveness in eliminating backdoors within self-supervised learning environments.
Most distillation-based defenses are designed for the supervised learning domain. For instance, NAD~\cite{2021-NAD} introduces a distillation-based defense in supervised learning, where a teacher model guides a poisoned student model to remove backdoors using attention distillation loss on a small subset of data. A notable aspect of this approach is that the teacher network is obtained through an independent fine-tuning process on the same small subset of the dataset.
Additionally, ARGD~\cite{xia2022eliminating} proposes a novel backdoor defense framework called Attention Relation Graph Distillation. This method explores correlations among attention features of varying orders by leveraging Attention Relation Graphs (ARGs), incorporating information from inputs, features, and outputs, and advancing the distillation-based defense methodology.
Furthermore, \cite{yoshida2020disabling} uses distillation to differentiate poisoned data from clean data, enabling developers to train a clean model from a poisoned dataset.
}

\section{Background}
\label{sec:background}
\begin{figure*}[tb]
    \centering
    \includegraphics[width=0.9\textwidth]{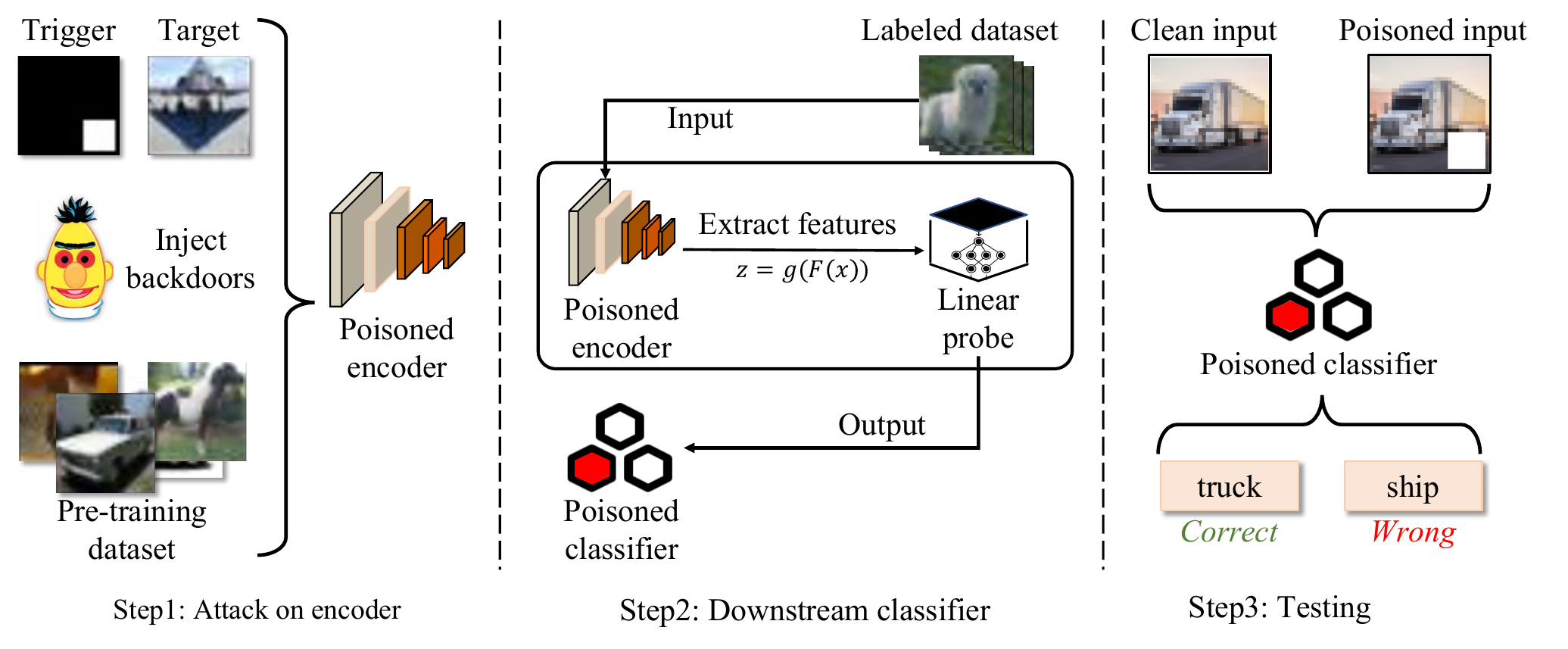}
    \caption{\revise{Backdoor attack against pre-trained encoders. Firstly, an attacker injects backdoors into an encoder and releases the poison encoder online, e.g., Hugging Face. Secondly, a user trains a classifier built on the backdoored encoder for a downstream task. During inference, the classifier built on the backdoored encoder has high accuracy on clean inputs but misclassifies inputs with the trigger as the attacker-chosen target.}
}
    \label{fig:encoder_attack_workflow}
\end{figure*}

In this section, we first introduce the background of self-supervised learning. Then, we explain the scenarios of backdoor attack and defense in self-supervised learning. At last, we review the concept of mutual information.

\subsection{Self-supervised learning}
\label{subsec:contrastive_learning}
Self-supervised learning(SSL) is a learning paradigm where models are trained using the inherent structure of data without explicit supervision, such as manually annotated labels.
There are multiple approaches to SSL, among which contrastive learning is a specific technique that aims to learn representations by contrasting positive and negative pairs of data samples, revolutionizing the field of SSL.
A large amount of related work~\cite{2020-moco,simclr,2020-certified} has emerged rapidly, showing great promise in achieving state-of-the-art results in various applications, such as image classification~\cite{2021-CLIP}.
Contrastive learning endeavors to optimize an encoder with the objective of generating similar feature embeddings for distinct augmented variations of the same input image (positive pair) while concurrently ensuring the production of dissimilar feature embeddings for different input images (negative pair).
To train such an encoder, contrastive loss defined by infoNCE~\cite{oord2018representation} between a positive pair of examples $(i, j)$ is introduced to search for optimal parameters~\cite{simclr}.
\begin{equation}
    \mathcal{L}_{i,j}=-log(\dfrac{exp(sim(z_i,z_j)/ \tau)}{\sum_{k=1}^{2N}{\mathbb{I}(k \neq i) \cdot exp(sim(z_i, z_k)/ \tau)}})
    \label{eq:contrastive_loss}
\end{equation}
where $z_i$ denotes the latent representation of view $x_i$, $sim(\cdot, \cdot)$ the similarity function, $\mathbb{I} \in \{0, 1\}$ an indicator function evaluating to 1 iff $k \neq i$ and $\tau$ a temperature parameter. 
Roughly speaking, contrastive loss enables pre-trained encoders to generate comparable feature vectors for diverse augmented versions of the same input image while concurrently ensuring that feature vectors for different input images are dissimilar.
\revise{These augmented versions of the input image are created through standard augmentation techniques (e.g., horizontal flipping, color jittering, and grayscale conversion) applied to the clean data.}
With such a training schedule, a well-trained pre-trained encoder offers numerous benefits.
\textbf{1).} A lower budget to train such an encoder because it's trained on uncurated and noisy data (easily collected from the Internet), which is much cheaper than that human-annotated~\cite{jaiswal2020survey,2022-Backdooring-Contrastive-Learning}.
\textbf{2).} A more generalized encoder because such models have learned features from a large amount of data, which makes them more generalized on multiple downstream tasks~\cite{simclr}.
The encoders mentioned in this paper are all pre-trained by contrastive learning.

\subsection{Backdoor attack on self-supervised learning}
\label{subsec:threat_model}
Backdoor attack on a deep learning model causes inputs stamped with a predefined trigger pattern (known as the trigger) to be misclassified into a specific target class (referred to as the attack target)~\cite{2024-Backdoor-Learning-A-Survey, 2018-Trojaning-Attack, 2023-Survey-Backdoor-Attacks-Countermeasures-in-DNN}.
Backdoor attacks are widely prevalent across multiple domains, such as image~\cite{2017-BadNets}, natural language~\cite{2021-BadNL}, programming language~\cite{2023-BadCode}, and so on.  

\noindent\textbf{Threat model.}
\revise{
Our threat model is pre-trained encoders that are unverified but publicly accessible.
An attacker can compromise a pre-trained encoder during the pre-training phase through either data poisoning (injecting maliciously crafted samples with triggers into the training dataset) or model poisoning (manipulating the training process).
The attacker’s objective is to implant backdoors that cause the encoder to misclassify any input with a specific trigger into a target label during downstream tasks.
The defender, on the other hand, receives the potentially backdoored encoder from an untrusted third party and has no knowledge of the attacker's trigger pattern or the intended downstream tasks. 
However, the defender assumes access to a small, verified, clean dataset (5\% of the pre-training data) that can be used to guide the defense process. The goal of the defender is to mitigate the backdoor while ensuring the encoder maintains high accuracy on clean data and significantly reduces attack success rates (ASR). 
}

Figure~\ref{fig:encoder_attack_workflow} illustrates a typical backdoor attack against pre-trained encoders.
The attacker in Figure~\ref{fig:encoder_attack_workflow} takes \textit{ship} as the target label, implanting backdoors in a clean encoder through data-poison or model-poison methods.
After encoder poisoning and downstream classifier training, the classifier tends to predict
the label of the attack target when the trigger is present.
As a result, a clean truck image can be correctly predicted by the classifier, whereas a truck image stamped with the trigger is classified as the target label \textit{ship}.
Our approach is implemented during the phase of poisoned encoders prior to their utilization for classifier training.

\subsection{Backdoor defense on self-supervised learning}
\label{subsec:backdoor_defense_ssl}
In contrast, defenders aim to remove the potential backdoor in encoders while guaranteeing performance similar to clean ones.
Roughly speaking, the targets of defenders can be categorized into two facets, \textit{effectiveness} and \textit{security}.

\noindent\textbf{Defender targets.}
\revise{
Defenders aim to purify an encoder that enables downstream classifiers to classify images based on their semantics rather than the attackers' intentions. From a metrics perspective, the objectives are to achieve a low attack success rate, referred to as \textit{security}, and high accuracy in downstream tasks, referred to as \textit{effectiveness}.
}

Facing serious security threats, numerous studies have already been proposed in the supervised domain.
Some attempt to reconstruct the backdoor trigger to counteract its influence using the inversion trigger~\cite{2019-Neural-Cleanse, 2022-MOTH, 2023-UNICORN}, while others resort to knowledge distillation with the goal of condensing the model's benign aspects and eliminating the malicious elements~\cite{2021-NAD, 2020-disable, papernot2016distillation}.
\revise{However, our further studies reveal that these defense approaches face major challenges when applied to SSL, primarily due to the absence of labels and the task-agnostic nature of pre-trained encoders. To overcome these limitations, we propose \ours{}, a novel framework that uses mutual information to identify and extract benign knowledge from the encoder. This approach is specifically designed for SSL, offering a task-agnostic and label-independent solution.}

To facilitate the description later, we define notions at first.
Let $\mathcal{F}$ denote a feature extractor, which can be represented as a function $\mathcal{F}: \mathcal{X} \rightarrow \mathcal{D}$, where $\mathcal{X}$ is the input space and $\mathcal{D}$ the embedded feature space. 
A feature extractor usually consists of a sequence of $n$ layers that are connected as follows.
\begin{equation}
    \mathcal{F}_\theta (x) = \mathcal{FC} \circ \mathcal{F}_\theta^{n-1} \circ \mathcal{F}_\theta^{n-2} \cdots \circ \mathcal{F}_\theta^0 (x),
\end{equation}
where $\mathcal{F}_\theta^0$ is the first layer and $\mathcal{F}_\theta^{n-1}$ the last. 
$\mathcal{FC}$ means the fully connected layer mapping high-dimensional feature maps to latent features, which are utilized in downstream tasks.
Variable $\theta$ denotes weight parameters in the $n$ layers.

\subsection{Mutual information}
\label{subsec:mi}
Mutual information holds significant relevance in the domain of statistical learning as it serves as a fundamental measure of the correlation between two random variables, encompassing both linear and nonlinear correlations~\cite{2018-MINE}. Let $X$ and $Z$ represent two random variables, the mutual information between $X$ and $Z$ can be understood as the decrease of the uncertainty in $X$ given $Z$:
\begin{equation}
    I(X; Z) := H(X) - H(X|Z)
    \label{eq:mutual_information}
\end{equation}
where $H(X)$ is the Shannon entropy, and $H(X|Z)$ is the conditional entropy of $X$ given $Z$.
Intuitively, from Equation~\ref{eq:mutual_information}, mutual information indicates the intensity of the dependence between $X$ and $Z$. 
As the value increases, so does the strength of their interdependence. 
In our scenario of backdoor threats, we utilize mutual information to measure the relevance between each layer's output and the extracted features of pre-trained encoders.
Combined with the observation as shown in Figure~\ref{fig:motivation_clone}, \ours{} utilizes mutual information as guidance to locate benign knowledge of a given backdoored encoder and deploy the distillation process.
\revise{``Benign knowledge’’ refers to the useful and unperturbed features or representations within a neural network that contribute to its intended functionality without being influenced by malicious triggers or backdoor behaviors. These features are derived from clean data and are essential for achieving high performance on downstream tasks. Mathematically, let $\mathcal{Z}$ represent the set of all features extracted by the encoder, and $\mathcal{Z}_{\text{benign}} \subseteq \mathcal{Z}$ denote the subset of features associated with clean data. The mutual information (MI) between clean inputs $X_{\text{clean}}$ and the encoder outputs $\mathcal{Z}$ is maximized when $\mathcal{Z}_{\text{benign}}$ is isolated, as expressed by: \[ \mathcal{Z}_{\text{benign}} = \arg\max_{\mathcal{F}'} I(X_{\text{clean}}; \mathcal{Z}_{\mathcal{F}'}) \] where $\mathcal{Z}_{\mathcal{F}'}$ represents the features extracted by the encoder using $\mathcal{F}'$, and $I(\cdot; \cdot)$ is the mutual information. This ensures that the identified features are both relevant to the clean data and free from backdoor influences.}
\revise{
In the presence of a backdoor, mutual information (MI) between layers is disrupted. Poisoned data activates sparse, localized neurons associated with the backdoor, preventing effective information flow across layers. As a result, MI remains low and does not exhibit the progressive accumulation observed with clean data. Figure~\ref{fig:motivation_mutual_information} supports the empirical evidence.
}

\section{Motivation}
\label{sec:motivation}
Backdoor mitigation techniques are extensively studied for classification tasks but lack in SSL.
We adapt a few well-known approaches to removing backdoors in pre-trained encoders and study their limitations.
\revise{\ours{} is specifically designed to address these limitations by adopting a novel perspective: leveraging mutual information to identify and extract benign knowledge from backdoored pre-trained encoders. This approach ensures effective backdoor mitigation while preserving the core functionalities of the encoder, thereby achieving the defender's objectives of maintaining high accuracy and robustness against attacks.}

\smallskip
\noindent\textbf{Limitation of Trigger Inversion-based Methods.}
Trigger inversion-based techniques as referred in~\cite{2022-MOTH,2019-Neural-Cleanse} have to select specific class pairs and craft an optimization algorithm to search for an input pattern that yields the desired attack effect, namely inducing misclassification to the target label.
After the inversion, existing techniques attempt to counteract the influence of the injected trigger by unlearning the inverted pattern.
Nevertheless, within SSL, defenders lack knowledge of target labels or downstream tasks making crafting optimization a problem. 
Recently, a novel method DECREE~\cite{2023-DECREE} overcomes the existing limitations using pair-similarity to craft the optimization.
Notwithstanding the efficacy of inverted triggers obtained by DECREE in detecting backdoored encoders, they do not suffice to fully eradicate the influence of the backdoor as shown in Table~\ref{tab:overall_performance_of_defenses}.
We ascribe it to the inversion quality.
The inverted triggers depend on the guidance selection, where pair-label is significantly better than pair-similarity.
Futhermore, as a one-round inversion, the poisoned encoder stays static in the process of DECREE, limiting the inverted triggers' stability and abundance.
\revise{These limitations result in low-quality inverted triggers, which allow DECREE to differentiate between poisoned and clean encoders but significantly reduce its effectiveness in removing backdoors within SSL due to the lack of access to labeled data.}

\begin{figure}[t]
    \centering
    \includegraphics[width=0.6\columnwidth]{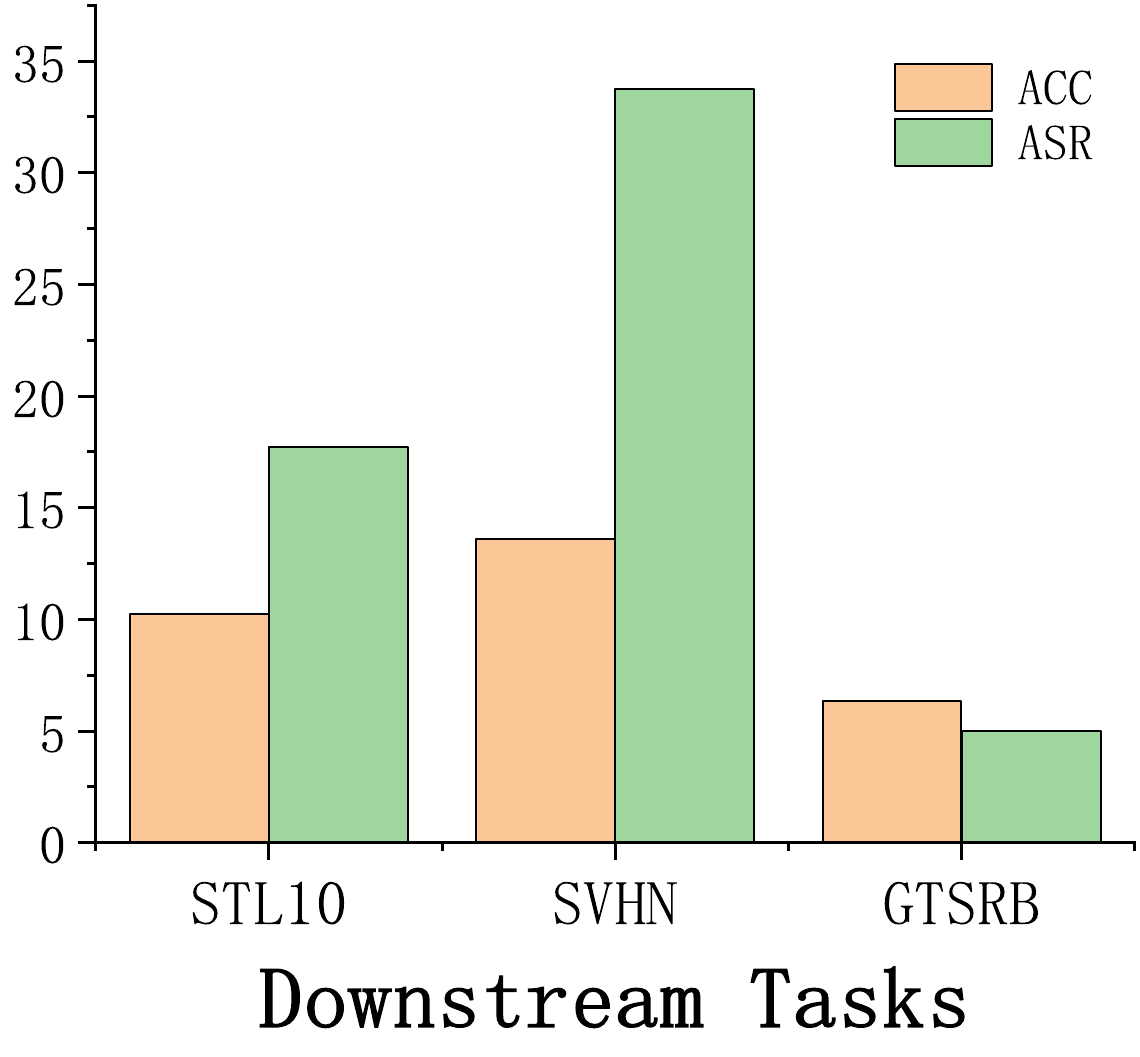}
    \caption{The performance of teacher nets.}
    \vspace{-12pt}
    \label{fig:teacher_net_performance}
\end{figure}

\noindent\textbf{Constraints of Knowledge Distillation (KD)-based methods in SSL.}
KD-based methods typically comprise three components: teacher net, student net, and distillation loss~\cite{2015-distill}.
Following the KD-based framework like NAD ~\cite{2021-NAD}, we directly adopt it in SSL but find that NAD falls short of attaining the desired security objective.
NAD tries to neutralize student net's malicious neurons through a carefully crafted distillation loss on the level of layers. 
The detoxification outcome relies greatly on the property of the teacher net. 
Nevertheless, following the guidance in~\cite{2021-NAD}, we fine-tune backdoor encoders on a small set of clean data as teacher nets but find the performance not satisfactory.
As shown in Figure~\ref{fig:teacher_net_performance}, the green bar represents the attack success rate on different downstream tasks while the orange bar the accuracy sacrifice.
It is evident that teacher nets in SSL excessively tailor their performance on a limited clean dataset due to fine-tuning causing \textit{catastrophic forgetting}~\cite{2017-catastrophic-forgetting} and inherits the raw encoders' backdoor behavior. 
We posit that this inheritance is the underlying cause for the diminished efficacy of NAD~\cite{2021-NAD}, leading to a loss of its desirable attributes.

\smallskip
\noindent\textbf{Insight of \ours{}}.
To mitigate the influence of malevolent neurons within pre-trained encoders entirely, we propose the adoption of an empty student solution, as referred in~\cite{2015-distill,2023-MEDIC,beyer2022knowledge,ma2023multi}.
The key intuition behind this is that malicious neurons stay dormant when fed clean samples.
\revise{
Figure~\ref{fig:motivation_clone} builds on this intuition within the context of SSL. From this perspective, several neuron-based studies~\cite{2021-ANP, 2018-FP, 2023-MEDIC} attempt to distinguish between malignant and benign neurons. However, our research makes an important observation: due to the lack of a reliable framework for qualitatively identifying neurons, these techniques often prove to be unstable and highly dependent on specific parameter settings.
}
Fortunately, our research makes a noteworthy discovery, revealing that the benign knowledge of pre-trained encoders exhibits distinctive separation properties concerning mutual information in layers.
As an illustrative instance, we consider ResNet-18~\cite{ResNet} as our default model. 
In this architecture, each layer takes the outputs of the preceding layer as its input and provides its outputs to the subsequent layer, thereby establishing a linear forward propagation structure.
We estimate mutual information(MI) values~\cite{2018-MINE} $I(\mathcal{F}_\theta^l(x), z)$ for each layer, where  $\mathcal{F}_\theta^l(\cdot)$ means the outputs of $l$-th layer and $z$ the latent extracted features of the entire encoder.
\revise{Figure~\ref{fig:layer-wise_outputs} illustrates the data flow inside the encoder. 
The input data $\mathcal{X}$ flows through the encoder $F$ ($F^1_{\theta}$ to $F^N_{\theta}$), producing intermediate outputs.
It finally extracts the features $\mathcal{Z}$.
Specifically, $z$ is defined as:
\begin{equation*}
    z = g(\mathcal{F}(x)), x \in \mathcal{X}
\end{equation*}
where $\mathcal{F}(\cdot)$ is the last layer's output of the encoder, and $g(\cdot)$ denotes the projector to extract the final latent features following~\cite{simclr,2020-moco}.
}
Figure~\ref{fig:motivation_mutual_information} intuitively illustrates the changes of MI on a poisoned encoder, where the blue line shows the MI between $\mathcal{F}_\theta^l(\cdot)$ and $z$ when fed clean images; the red line shows the MI between each layer's outputs and extracted poisoned features when fed poisoned images. 
\revise{
We can define the extracted poisoned features as follows:
\begin{equation*}
z_{\text{poisoned}} = g(\mathcal{F}(x_{\text{poisoned}})), x_{\text{poisoned}} \in \mathcal{X}
\end{equation*}
where $x_{\text{poisoned}}$ represents the poisoned image input, $\mathcal{F}(\cdot)$ is the last layer's output of the encoder, $g(\cdot)$ denotes the projector to extract the final latent features.
}

\begin{figure}[t]
    \centering
    \includegraphics[width=0.7\columnwidth
    ]{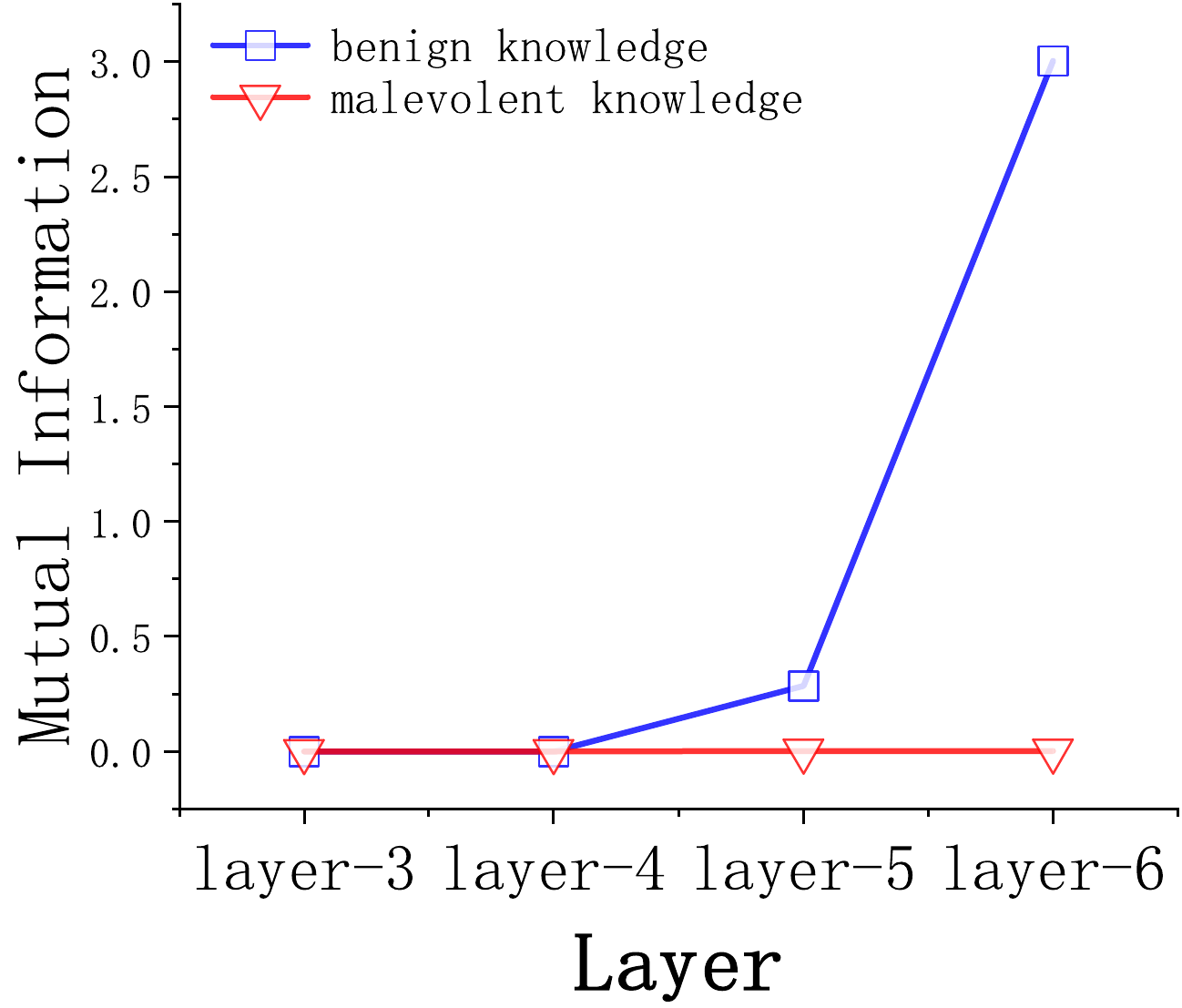}
    \caption{Mutual information guidance.\revise{It showcases the mutual information between each layer's output and the final latent extracted features.}}
    \vspace{-8pt}
    \label{fig:motivation_mutual_information}
\end{figure}
\begin{figure}[t]
    \centering
    \includegraphics[width=0.99\columnwidth
    ]{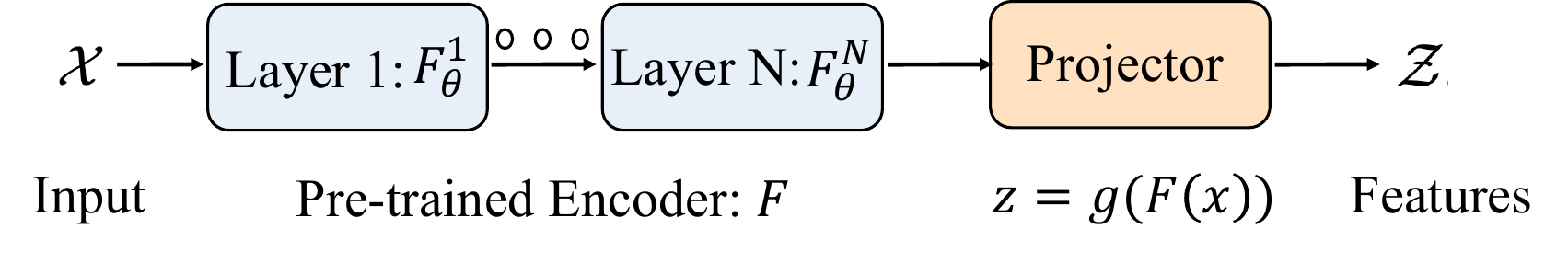}
    \caption{Layer-wise outputs: With $\mathcal{X}$ as the input and $\mathcal{Z}$ as the final extracted features, the diagram highlights the data flow through each layer of the encoder. The classification linear probe is trained on $\mathcal{Z}$, utilizing the final extracted features for downstream tasks.}
    \vspace{-12pt}
\label{fig:layer-wise_outputs}
\end{figure}
Observe that as the number of layers increases, the MI of benign knowledge increases, while the MI of poisoned knowledge stays negligible. 
The gap is more evident in the deepest layer. 
Based on this, we have the following intuition, which inspires us to design \ours{} aiming to mitigate backdoor influence by transferring benign knowledge out:
\begin{intuition}[The chain property of benign features] 
\label{intution:chain_benign_features}
The benign features that enable a given encoder's effectiveness are mostly hidden in the last layer.
\end{intuition}
This intuition makes it possible to locate benign knowledge of poisoned encoders and transfer it to clean ones.
Theoretically, such a chain property is attributed to a successive Markov chain of intermediate representations generated by the encoder's layered structure. 
With clean inputs, only benign knowledge is extracted.
Each layer in the encoder can be quantified by the amount of information it retains on the inputs and the extracted features.
We theoretically explain our intuition by proving the following theorem.
\begin{theorem}[Markovian property of benign features]
   $I(\mathcal{F}_\theta^{0}(x),z)\leq I(\mathcal{F}_\theta^{1}(x),z)\leq\cdots\leq I(\mathcal{F}_\theta^{n-2}(x),z)\leq I(\mathcal{F}_\theta^{n-1}(x),z)$, where $z$ denotes the final extracted features by pre-trained encoders and $\mathcal{F}_\theta^l(\cdot)$ the outputs of $l$-th layer.
\label{theo:markov_benign_features}
\end{theorem}
\textit{Remarks.} 
With Theorem~\ref{theo:markov_benign_features}, the mutual information between extracted features and the last layer is the largest, indicating that benign features are mostly hidden in the last year. 
The proof details can be found in Appendix~\ref{sec:appendix}.

\section{Methodology}
\label{sec:design}

\begin{figure*}[tb]
    \centering
    \includegraphics[width=0.95\textwidth]{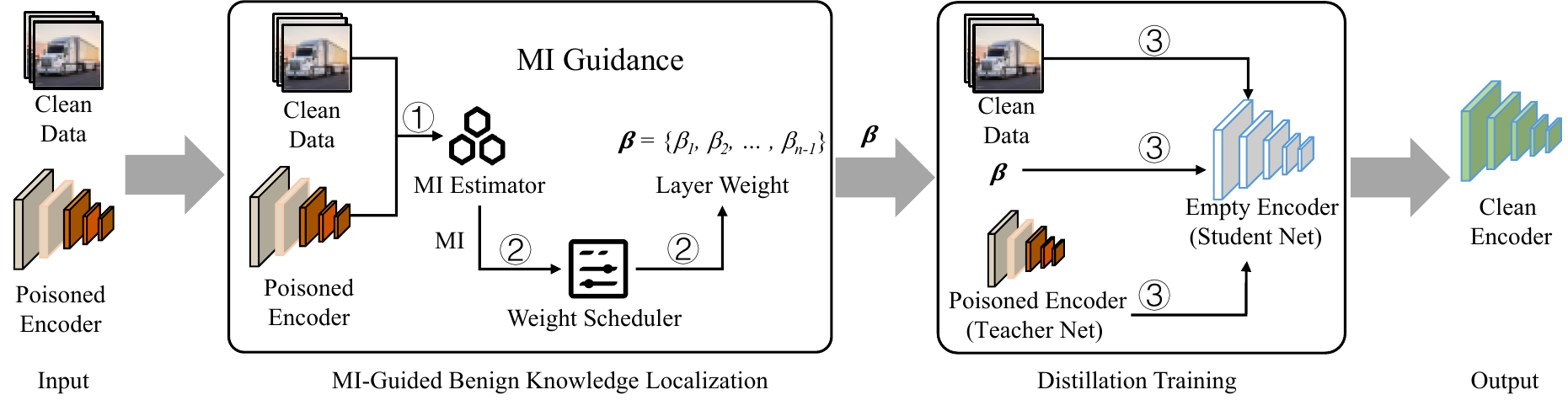}
    \caption{The framework of \ours{}.}
    \vspace{-6pt}
    \label{fig:framework}
\end{figure*}

Figure~\ref{fig:framework} illustrates the overview of \ours{}.
Given a poisoned encoder and a small set of clean data, \ours{} decomposes the backdoor mitigation process into two phases: mutual information-guided benign knowledge localization and distillation training.
In the first phase, \ours{} introduces a mutual information (MI) estimator to estimate the MI between each layer and the final extracted features.
We adopt \textit{MINE}~\cite{2018-MINE} to train the MI estimator, which treats MI  as the Kullback-Leibler(KL) divergence~\cite{kullback1997information} and converts it into the dual representation, outputting the estimated MI lower bound. 
Then, \ours{} utilizes a weight scheduler to convert the estimated MI to a layer weight $\bm{\beta}$ representing the distillation potency (\circled{2}). Finally, the MI representing benign knowledge locations is passed as the layer weight to guide the distillation training in the next phase.
In the second phase, \ours{} creates an empty encoder $\mathcal{F}'$ as the student net, which has the same architecture as the poisoned encoder $\mathcal{F}$ but initialized with random weights.
With $\mathcal{F}$ as the teacher net, a small clean set, and the layer weight, the distillation training aims to convert the student net (i.e., an empty encoder) to a well-trained clean encoder (\circled{3}).
Roughly speaking, \ours{} transfer knowledge from $\mathcal{F}$ to $\mathcal{F}'$ by closing the distance between their extracted features and standard contrastive loss.
\revise{
To better capture the benign neuron behaviors from the teacher network, \ours{} aligns the convergence of attention maps generated by different layers, with each layer weighted by the vector $\bm{\beta}$. The role of $\bm{\beta}$ is specifically defined in Equation~\ref{eq:security_loss}, where it weights the attention loss to enhance alignment across distinct layers.
}

\subsection{MI-Guided Benign Knowledge Localization}
\label{subsec:mi_guidance}
To distill an effective and benign model, the first step is to locate where the benign knowledge resides. 
Note that the distillation process is performed at the layer level, and different layers contain different amounts of benign knowledge.
For example, as expounded in Section~\ref{sec:motivation}, the benign knowledge predominantly congregates within layers proximate to the output.
Therefore, we make a finer division of the distillation strength for distinct layers, aiming to locate where the benign knowledge resides.
Specifically, we formulate the mutual information guidance to identify the location of benign knowledge and \revise{identify varying degrees of distillation potency across distinct layers.} 
This strategic allocation contributes to \ours{} extracting benign neuron behaviors while concurrently eliminating the presence of backdoors.

\begin{algorithm}
\small
  \caption{MI Guided Benign Knowledge Localization}
  \label{alg:weight_schedule}
  \begin{algorithmic}[1] 
    \REQUIRE teacher net $\mathcal{F}$ and a small clean set $\mathcal{X}$
    \ENSURE Distillation potency weights $\bm{\beta}=\{\beta_0, \beta_1, \dots, \beta_{K-1}\}$ 
    \STATE Set base weight $\alpha_0$, optimal weight $\alpha_1$ \hfill \COMMENT{$\alpha_0, \alpha_1$ are constants}
    \STATE $\bm{m}:\{m_0, m_1, \dots, m_l, \dots, m_{K-1} \} \gets \bm{0}$
    \FOR {$l$ {\bf in} $0 \dots K-1$}
        \STATE Random init mutual information estimator $T$
        \STATE $T \gets MINE(T, \mathcal{F}(\mathcal{X}), \mathcal{F}^l(\mathcal{X}))$
        \STATE $m_l \gets T(\mathcal{F}(\mathcal{X}), \mathcal{F}^l(\mathcal{X}))$
    \ENDFOR
    \STATE $\bm{\beta} \gets \alpha_0+\alpha_1*\dfrac{\bm{m} - \overline{\bm{m}}}{\sigma(\bm{m})}$
    \RETURN $\bm{\beta}$
  \end{algorithmic}
\end{algorithm}

Algorithm~\ref{alg:weight_schedule} presents the implementation details of MI-guided benign knowledge localization, \revise{which treats the poisoned teacher net} $\mathcal{F}$ and a small clean set $\mathcal{X}$ as input, and the output $\beta$ hints for the locations of benign knowledge.
\revise{
A high value of $\beta$ indicates that the corresponding layer likely contains more benign knowledge and requires a higher distillation potency. Conversely, a low value of $\beta$ suggests potential influence from backdoor behavior, requiring less emphasis in the distillation process.
}
Specifically, \ours{} first defines two constants $\langle \alpha_0, \alpha_1 \rangle$ to assign the distillation potency, where $\alpha_0$ determines the basic distillation performance and $\alpha_1$ is dynamically allocated in different layers by mutual information (line 1). Then, \ours{} defines a vector $\bm{m}$ initialized by $\bm{0}$ to save the estimated mutual information where $m_l$ stores $l$-th layer's mutual information and the extracted features (line 2).
\ours{} estimates the mutual information between $\langle \mathcal{F}^l(\cdot), \mathcal{F}(\cdot) \rangle$ layer by layer (lines 3--6), where $\mathcal{F}^l(\cdot)$ denote the outputs of the $l$-th layer and $\mathcal{F}(\cdot)$ the final extracted features, respectively.
The MI estimation algorithm we adopt is $MINE$~\cite{2018-MINE}, which utilizes a neural net to estimate the lower bound of MI (lines 4--6).
Once the MI values on different layers are obtained, we normalize them as the standard to assign specific weights to each layer (line 8), culminating in the algorithm's termination and the consequent return of the layer weights denoted as $\bm{\beta}$ (line 9).

\revise{\noindent\textbf{Algorithmic complexity analysis.} The complexity of Algorithm~\ref{alg:weight_schedule} is dominated by the mutual information (MI) estimation step, which computes MI for each of the $K$ layers in the teacher network. For each layer, MI estimation involves processing the clean dataset $\mathcal{X}$ with $n$ samples and feature dimensionality $d$, resulting in a per-layer complexity of $O(n \cdot d)$. Since this is repeated across $K$ layers, the overall complexity of the MI estimation step is $O(K\cdot n \cdot d)$. 
Consequently, the total time complexity of Algorithm~\ref{alg:weight_schedule} is $O(K \cdot n \cdot d)$, with MI estimation being the most computationally expensive operation.
}

\subsection{Distillation Training}
\label{subsec:distillation_training}
\revise{
The goal of distillation training is to produce a clean encoder, satisfying the effectiveness and security requirements of the given poisoned encoder.} 
To train an effective student model, distillation training typically clones as complete knowledge as possible from the teacher model. In our scenario, as shown in Figure~\ref{fig:framework}, to train an effective encoder (student net), we design a \textit{clone loss} to clone as much knowledge as possible from the poisoned encoder (i.e., teacher net).

\noindent\textbf{Clone~loss}.
To achieve our distillation objectives, \ours{} adopts an empty encoder as the student net, ensuring strong plasticity.
In order to accomplish our student net's effectiveness, distillation-based techniques distill high-level knowledge from the teacher net and transfer it to the student net.
Considering  the student net is an empty encoder, we design \textit{clone loss} specifically to clone the teacher net's representational capacity,  comprising two distinct terms:
\begin{align}
    \label{eq:L_0} 
    & L_0 = cosine(\mathcal{F}(x), \mathcal{F'}(x)) \\
    \label{eq:L_1} & L_1 = CLS(x, \mathcal{F'})
\end{align}
\revise{
where $cosine(\cdot)$ represents the cosine distance function. $\mathcal{F}(x)$ and $\mathcal{F}'(x)$ denote the features extracted from the input $x$ by the teacher encoder $F$ and the student encoder $F'$, respectively. $CLS(\cdot)$ refers to the standard contrastive learning loss, as described in Section~\ref{subsec:contrastive_learning}.
}

As declared in~\cite{2023-MEDIC, 2021-NAD, xia2022eliminating}, the malicious neurons stay dormant when clean samples are fed to classifiers.
We extend this observation in the scenario of pre-trained encoders.
\revise{
Figure~\ref{fig:motivation_clone} illustrates attention maps generated by clean and poisoned encoders when processing clean samples, with Grad-CAM~\cite{2017-Grad-cam} employed for visualization. Grad-CAM highlights the regions in the input image that the encoder considers most important. In Figure~\ref{fig:clean_enc_poison_sample}, the clean encoder focuses on the central area of the sample, indicating its emphasis on relevant features. Similarly, in Figure~\ref{fig:poisoned_enc_clean_sample}, the poisoned encoder displays almost identical attention patterns when applied to the same clean sample. This similarity demonstrates that the poisoned encoder behaves like the clean encoder in the absence of backdoor triggers, as only benign functionalities are activated by clean inputs.
}

This observation makes it possible to clone benign knowledge from a poisoned encoder to the student net.
In that case, $L_0$ is designed to achieve this target helping pre-trained encoders to extract benign features in downstream tasks. 

\begin{figure}[htbp]
\centering
    \subfloat[Clean encoder on clean sample]
    {
        \includegraphics[width=0.376\linewidth]{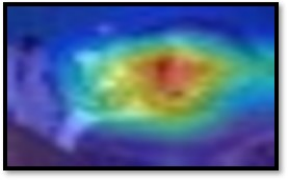}
        \label{fig:clean_enc_poison_sample}
    }
    \hspace{7mm}
    \subfloat[Poisoned encoder on clean sample]
    {
        \includegraphics[width=0.376\linewidth]{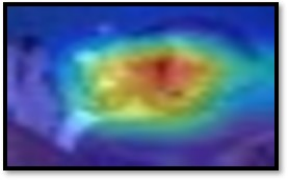}
        \label{fig:poisoned_enc_clean_sample}
    }
    \caption{Comparison of attention maps from clean and poisoned encoders on the same clean sample, visualized using Grad-CAM. Observe that both encoders exhibit nearly identical attention patterns, focusing on the same regions of the input.}
    \label{fig:motivation_clone}
\end{figure}
However, unfortunately, cloning the complete knowledge of the poisoned encoder will clone the backdoor knowledge into our student encoder at the same time, inheriting its backdoor behavior. For example, the model distilled with only the clone loss has an ASR of 69.04\% (detailed in Section~\ref{subsec:evaluation_results}--Answer to \textbf{RQ4}). 
This is attributed to the fact that our clone loss relies on the complete extracted features rather than individual neuron weights.
The cloning procedure spans the entirety of the encoder, consequently resulting in a degree of backdoor inheritance.
To avoid this phenomenon, inspired by~\cite{2017-attention-transfer,2021-NAD}, we introduce attention maps to represent encoder knowledge behind each layer.
Besides using attention maps at the layer level to minimize the disparity between the teacher encoder and student encoder, we also utilize them to compel the student encoder to prioritize benign knowledge and mitigate the presence of inherited malignant components. 
Specifically, we devise the following \textit{attention loss} to achieve this.

\noindent\textbf{Attention loss}.
\begin{equation}
    L_2 = \sum_{l=0}^{K-1}{\beta_l \left|\left| \dfrac{\mathcal{L}_{AT}(\mathcal{A}^l)}{||\mathcal{L}_{AT}(\mathcal{A}^l)||_2} - \dfrac{\mathcal{L}_{AT}(\mathcal{A'}^l)}{||\mathcal{L}_{AT}(\mathcal{A'}^l)||_2} \right| \right|_2}
    \label{eq:security_loss}
\end{equation}
where $||\cdot||_2$ means the $L_2$ norm, $K$ the number of layers. $\bm{\beta} = \{\beta_0, \beta_2, \dots, \beta_{K-1}\}$ is a vector that determines the distillation potency of each layer.
\noindent $\mathcal{L}_{AT}: \mathbb{R}^{C \times H \times W} \xrightarrow{} \mathbb{R}^{H \times W} $ is the attention operator mapping an activation map to an attention representation, which is formalized as:
\begin{equation}
     \mathcal{L}_{AT}(\mathcal{A}^l) = \sum_{i=1}^{C}{|\mathcal{A}^l_i|^p} 
\end{equation}
where $\mathcal{A}^l_i$ is the activation map of the $i$-th channel in $l$-th layer extracted by $\mathcal{F}$, $\mathcal{A}'^l$ of $\mathcal{F}$, $| \cdot |$ denotes the absolute value function and $p$ means the order amplifying the disparities between the backdoored neurons and the benign neurons.

As articulated in~\cite{2021-NAD}, this attention loss is thoughtfully devised to align neurons that exhibit higher responsiveness to the trigger pattern with benign neurons responsible solely for meaningful representations. This alignment strategy is instrumental in diminishing the overall impact of the trigger effects, thereby enhancing the capability to address the security threat posed by backdoors.

\noindent\textbf{Optimization problem}. Following the establishment of the definitions for the three loss terms $L_0$, $L_1$, and $L_2$, as defined above, we formulate the training of an effective and secure encoder as an optimization problem:
\begin{equation}
    \min_{\mathcal{F}'} \mathcal{L} = L_0 + \lambda_1 \cdot L_1 + \lambda_2 \cdot L_2
    \label{eq:opt_loss}
\end{equation}
where $\lambda_1$ and $\lambda_2$ are two hyperparameters to balance these three loss terms. 
At first, \ours{} initializes a student net with the same architecture as the provided pre-trained encoder, albeit with random parameters.
Then, \ours{} solves the optimization problem using gradient descent.

\section{Evaluation}
\label{sec:evaluation}

\begin{table*}[t]
    \centering
    \footnotesize
    \tabcolsep=2.2pt
    \renewcommand{\arraystretch}{1.4} 
    \caption{Backdoor removal results by different defense techniques. Column \revise{``Pre-train''} denotes the pre-training dataset used for constructing the encoder. Column ``Downstream'' denotes the downstream classifier. The best results are in bold.}
    \begin{tabular}{ccccccccccccccccccccc}
        \toprule

        \multirow{2.5}{*}{Attack} &
        \multirow{2.5}{*}{\revise{Pre-train}} &
        \multirow{2.5}{*}{Downstream} &
        \multicolumn{2}{c}{Undefended} &
        \multicolumn{2}{c}{FT} &
        \multicolumn{2}{c}{FP} &
        \multicolumn{2}{c}{NAD} &
        \multicolumn{2}{c}{ANP} &
        \multicolumn{2}{c}{MOTH} &
        \multicolumn{2}{c}{DECREE} &
        \multicolumn{2}{c}{MEDIC} &
        \multicolumn{2}{c}{\ours{}} \\ 
        
        \cmidrule(lr){4-5} \cmidrule(lr){6-7} \cmidrule(lr){8-9} \cmidrule(lr){10-11} \cmidrule(lr){12-13} \cmidrule(lr){14-15} \cmidrule(lr){16-17} \cmidrule(lr){18-19} \cmidrule(lr){20-21}
        & & & ACC & ASR & ACC & ASR & ACC & ASR & ACC & ASR & ACC & ASR & ACC & ASR & ACC & ASR &ACC &ASR &ACC &ASR\\

        \hline
        
        \multirow{8}{*}{\rotatebox{90}{BadEncoder}} & \multirow{3}{*}{CIFAR10} & GTSRB  & 83.37 & 99.09 & 77.04 & 5.00 & 78.41 & 6.63 & 78.85 & 11.73 & 76.94 & 26.78 & 56.64 & 6.95 & 72.64 & 86.05 & 73.15 & 5.77 & \textbf{81.33} & \textbf{1.36} \\
        
        & & SVHN & 68.71 & 99.02 & 55.12 & 33.77 & 56.49 & 63.36 & 66.01 & 29.84 & 63.75 & 32.32 & 58.43 & 64.05 & 60.42 & 82.25 & 57.05 & 33.64 & \textbf{75.32} & \textbf{12.31} \\
        
        & & STL10 & 75.98 & 99.73 & 65.74 & 17.71 & 65.77 & 17.52 & 69.86 & 15.90 & 57.98 & 75.93 & 60.41 & 26.56 & 69.32 & 88.96& 60.72 & 9.86 & \textbf{72.22} & \textbf{8.86} \\
        
        \cline{2-21} 
        
        & \multirow{3}{*}{STL10} &  GTSRB  & 79.54 & 91.24 & 75.94 & 5.91 & 78.21 & 5.58 & 76.31 & 5.43 & 77.29 & 20.33 & 50.20 & 4.88 & 70.41 & 98.44 & 74.39 & 1.42& \textbf{76.71} & \textbf{0.75} \\
        
        & & SVHN & 63.29 & 99.93 & 54.07 & 32.33 & 57.12 & 21.31 & 56.18 & 38.00 & 68.94 & 27.18 & 58.10 & 32.28 & 57.02  & 21.54 & 54.77 & 24.54& \textbf{70.75} & \textbf{18.05} \\
        
        & & CIFAR10 & 86.91 & 97.03 & 81.16 & 10.55 & 81.85 & 11.54  & 81.67 & 11.25  & 70.89 & 12.43 & 71.50 & 12.85 & 82.49 & 91.40 & 74.62 & 8.48 & \textbf{84.31} & \textbf{7.51} \\
    
        \cline{2-21}
        
        & \multicolumn{2}{c}{Average} & 76.30 & 97.67 & 68.17 & 17.54 & 69.64 &20.99 &71.48 &18.69 &69.29 &32.49  & 59.21 &24.59  & 68.71 & 78.10&65.79 &13.95& \textbf{76.77} & \textbf{8.14} \\ 
        
        \hline
        \hline
        
        \multirow{8}{*}{\rotatebox{90}{BASSL}} & \multirow{3}{*}{CIFAR10} &  GTSRB & 80.09 & 67.72 & 76.73 & 7.08 & \textbf{77.90} & 9.34 & 76.32 & 6.77 & 74.82 & 5.90 & 60.24 & 29.14 &61.91 & 12.94& 74.88 & 6.63& 76.86 & \textbf{3.19} \\
        
        & & SVHN & 66.08 & 80.03 & 59.72 & 30.91 & 61.35 & 39.85 & 61.99 & 37.93 & 66.70 & 15.90 & 69.55 & 17.80 & 55.04 & 26.97 & 57.06 & 12.53& \textbf{69.96} & \textbf{12.14} \\
         
        & & STL10 & 73.48 & 36.84 & 71.52 & 19.48 & 72.50 & 23.16 & 74.08 & 18.68 & 65.88 & 27.76 & 69.87 & 20.08 & 73.62 & 11.51 & 61.05 & 11.42 & \textbf{75.12} & \textbf{10.35}\\

        \cline{2-21} 
        
        & \multirow{3}{*}{STL10} &  GTSRB& 78.12 & 39.13 & 78.06 & 4.80 & 78.11 & 14.42 & 74.71 & 5.56 & 73.12 & 5.87 & 53.15 & 7.81 & 63.08 & 9.27& 75.56 & 5.40& \textbf{78.41}  & \textbf{5.23} \\
                             
        & & SVHN & 58.73 & 60.33 & 54.88 & 25.56 & 51.88 & 22.12 & 56.59 & 23.90 & 62.98 & 14.95 & 57.26 & 31.42 & 62.94 & 16.26 & 54.05 & \textbf{15.60}& \textbf{64.80} & 16.22 \\
                             
        & & CIFAR10  & 81.03 & 20.07 & 81.32 & 13.83 & 80.90 & 14.02 & 81.02 & 11.02 & 80.70 & 12.11 & 69.34 & \textbf{9.48} & 78.90 & 11.35 & 75.08 & 11.86& \textbf{81.61} & 11.96 \\
    
        \cline{2-21} 
        
        & \multicolumn{2}{c}{Average} &72.92 &50.68 &70.37 &16.94 &70.44 &20.48 &70.78 &17.31 & 70.70 & 13.74&63.23 &19.28 &65.91 &14.71 &66.28 &10.57& \textbf{74.46} & \textbf{9.84} \\ 
        \bottomrule
    \end{tabular}
\label{tab:overall_performance_of_defenses}
\end{table*}

From the view of a more rigorous evaluation, we undertake a series of comprehensive experiments to assess \ours{} across four distinct dimensions: effectiveness, robustness, generalization, and the examination of core components and hyper-parameters through ablation studies. 
We complete the assessment by answering the following research questions (\textbf{RQs}):
\begin{itemize}[noitemsep]
    \item\textbf{RQ1. The effectiveness of \ours{} on mitigating backdoors for pre-trained encoders.}
    \begin{itemize}
        \item\textbf{RQ1.1.} How effective is \ours{} in removing backdoors in SSL?
    \end{itemize}
    \item\textbf{RQ2. Ablation studies on core components and hype-parameter.}
    \begin{itemize}
        \item\textbf{RQ2.1.} How does each core component (including \textit{clone loss}, \textit{attention loss}, and weight scheduler) affect \ours{}? 
        \item \textbf{RQ2.2.} How do the hype-parameters $\lambda_1$ and $\lambda_2$ affect \ours{}?
    \end{itemize}
    \item\textbf{RQ3. The robustness of \ours{}.}
    \begin{itemize}
        \item\textbf{RQ3.1.} How does the trigger size affect \ours{}?
        \item\textbf{RQ3.2.} How does the clean data ratio affect \ours{}?
        \item\textbf{RQ3.3.} How does the poison ratio affect \ours{}?
        \item\textbf{RQ3.4.} What is the performance of \ours{} against adaptive attack?
    \end{itemize}
    \item\textbf{RQ4. The generalization of \ours{}.}
    \begin{itemize}
        \item\textbf{RQ4.1.} Will \ours{} continue its performance when extending to supervised learning? 
        \item\textbf{RQ4.2.} How does \ours{} react to clean encoders?
    \end{itemize}
\end{itemize}

\subsection{Experimental Setup}
\label{subsec:experimental_setup}
\noindent\textbf{Datasets and Models.}
In our evaluation, we utilize ResNet18~\cite{ResNet} as our default model.
All experiments are conducted on four widely-used datasets, including CIFAR-10~\cite{2009-CIFAR10}, STL-10~\cite{2011-STL10}, GTSRB~\cite{2012-GTSRB}, and SVHN~\cite{2011-SVHN}. 
Below are the details for the four datasets.
\begin{itemize}[topsep=0pt, itemsep=0pt, leftmargin=15pt]
     \item \textbf{STL10~\cite{2011-STL10}.}
     There are 105,000 training images, 5,000 of which are labeled while others are not, and 8,000 test images in this dataset.
     Each image has a size of 96x96x3 and belongs to 10 classes.
    \item \textbf{CIFAR10~\cite{2009-CIFAR10}.}
    There are 50,000 training images and 10,000 test images in this dataset, each of which is 32x32x3.
    10 classes in total.
    \item \textbf{GTSRB~\cite{2012-GTSRB}.} The dataset encompasses more than 50,000 images distributed across 43 categories, with each image 32×32×3.
    \item \textbf{SVHN~\cite{2011-SVHN}.} There are more than 70,000 training images and 20,000 test images of Google Street View to represent house numbers.  
    Each image has a size of 32×32×3 and is associated with one of the 10 digits.
\end{itemize}

\smallskip
\noindent\textbf{Attacks and setting.} We consider a series of \revise{state-of-the-art} attack techniques, \revise{all of which are designed for SSL specifically:} a) BadEncoder~\cite{2022-BadEncoder}, b) BASSL~\cite{BASSL}.
To verify \ours{}'s effectiveness in defending these attacks, we allow attackers to maximize their knowledge to realize attack targets. 
For example, when conducting BASSL, we migrate more than half of the downstream target class to inject backdoors.
To ensure equitable comparison, the utilization of triggers entails the employment of uniform white squares with dimensions of $10 \times 10$.

\smallskip
\noindent\textbf{Defender knowledge and settings.}
In a typical defense scenario, we assume that we have obtained an untrustworthy encoder from a third party, such as an outsourced training service.
To defend against the encoder's potential backdoors, we have a small set of clean data, 4\% of the pre-training data, specifically.
The objective of \ours{} is to replicate a reliable encoder from the unverified encoder without any hidden malicious behavior, and ensure that it maintains high accuracy when handling uncontaminated data.

\smallskip
\noindent\textbf{Baselines.} 
We compare \ours{} with seven existing defense methods, including standard Fine-Tuning (FT), Fine Pruning(FP)~\cite{2018-FP},  Neural Attention Distillation (NAD)~\cite{2021-NAD}, Adversarial Neuron Pruning (ANP)~\cite{2021-ANP}, Model Orthogonalization(MOTH)~\cite{2022-MOTH}, DECREE~\cite{2023-DECREE} and MEDIC~\cite{2023-MEDIC}.
In order to conduct a fair test, all defense methods are assumed to have equal access to a set of 2,000 clean training images, which constitutes 4\% of the available data. Note that all the above defenses were initially created for use in supervised learning, and we are adapting and applying them to the domain of self-supervised learning.

\noindent\textbf{Evaluation Metrics}
We follow existing works~\cite{2017-BadNets, 2022-DBD, 2022-MOTH, 2022-BadEncoder, BASSL} and use attack success rate (ASR) and clean accuracy (ACC) on downstream classifiers as the metrics.
ASR measures the percentage of trigger-injected inputs predicted as the target label by the downstream classifier.
ACC measures the classification accuracy of the downstream classifier on a clean test set.
Note that the attack and defense processes take place during pre-training while the assessment occurs downstream.

\subsection{RQ1:Evaluation Results on Effectiveness}
\label{subsec:evaluation_results}
\subsubsection{RQ1.1: How effective is MIMIC in removing backdoors in SSL}
\label{subsubsec:rq1_effectiveness}
A well-designed backdoor mitigation technique should remove backdoors in pre-trained encoders while maintaining the performance on clean data.
Considering pre-trained encoders are applied to specific downstream tasks, we deploy backdoor mitigation techniques during pre-training but conduct evaluation in downstream.
Given the present absence of dedicated defense techniques tailored specifically for SSL, we adopt seven SOTA defense techniques proposed for supervised learning as baselines, details of which are illustrated below.
In particular, DECREE~\cite{2023-DECREE} is an approach aiming to distinguish backdoored pre-trained encoders from clean ones through trigger inversion. 
In the context of backdoor mitigation, we employ an inversed trigger to nullify the trigger information by reducing the similarity between pristine images and those doctored with the inversed trigger. 

Table~\ref{tab:overall_performance_of_defenses} presents the effectiveness of different defense techniques in mitigating backdoors in self-supervised learning.
We conduct experiments on two distinct types of attacks BadEncoder~\cite{2022-BadEncoder} and BASSL~\cite{BASSL}. 
The experimental results for the former are presented in the top half of the table, while those of the latter are shown in the bottom half.
The column labeled ``Undefended'' delineates the outcomes achieved by classifiers constructed upon the backdoored encoders. 
Subsequent columns present the performance of classifiers built on repaired encoders by seven baselines and \ours{}.
Observe that \ours{} demonstrates the most significant decrease in ASR and the least amount of ACC loss.
Respectively, \ours{} exhibits no ACC loss while achieving an ASR reduction exceeding 89.53\% on the BadEncoder-attack, a performance pattern that sets it apart from all baselines. 
Taking the GTSRB classifier built on BadEncoder-attacked CIFAR10 encoder (the first row) as an example, the undefended classifier has an ASR of 99.09\%.
\ours{} reduces it to 1.36\%, whereas existing techniques can only reduce the ASR to 5.77\% as best.  
As for the second attack, note that the ASRs of classifiers built on BASSL are much lower than those by BadEncoder. 
This is consistent with the observation by~\cite{li2022demystifying}.
The original BASSL paper~\cite{BASSL} uses the number of false positives (misclassified samples) instead of ASR as the metric. 
Nevertheless, \ours{} is still able to eliminate those backdoors with the largest ASR reduction, from 50.68\% to 9.84\% on average. 
In certain instances, we discern that \ours{} maintains ASRs within the range of approximately 10\%-20\%. 
This phenomenon arises due to the fact that presenting these trigger-laden images to an untainted classifier also yields a comparable percentage of predictions matching the target label. 
But overall, \ours{} outperforms the seven baselines in backdoor removal impressively.
\revise{
In certain scenarios, such as CIFAR10 + GTSRB for BASSL attacks, \ours{} exhibits slightly lower accuracy (ACC) compared to some baseline methods. This can be attributed to the trade-off between backdoor mitigation and feature preservation during the distillation process. In datasets with complex and diverse feature distributions, such as GTSRB, the process of suppressing backdoor behaviors may inadvertently affect the retention of fine-grained features critical for downstream tasks, leading to a slight drop in ACC. However, it is important to note that \ours{} consistently achieves the lowest attack success rate (ASR) across all scenarios, indicating its robustness in effectively mitigating backdoors.
}

\subsection{RQ2. Evaluation Results on Ablation Studies.}
\label{rq2:ablation_studies}
\subsubsection{RQ2.1. How does each core component (including clone loss, attention loss, and weight scheduler) affect \ours{}}
\ours{} contains four core components \textit{clone loss} (denoted $L_0$ and $L_1$), \textit{attention loss} (denoted $L_2$), and weight scheduler guided by mutual information.
As declared in Section~\ref{sec:design}, \ours{} first creates an empty encoder as the student net to learn benign knowledge from a given poisoned pre-trained encoder, which is treated as the teacher net.
Then, \ours{} utilizes the weight scheduler to locate benign knowledge in teacher net and \textit{clone loss} to transfer it to student net.
In this process, \textit{attention loss} is deployed to ensure both teacher and student models pay attention in a similar way to clean data.
Experiments are conducted on CIFAR10 (the pre-training dataset) and GTSRB (the downstream task).
We explore the influence of each component on \ours{} by controlling variables, observing whether they function as intended.

Figure~\ref{fig:ablation_core_components} reports the experimental results, where lines denote the ACC and bars the ASR.
Observe that only \textit{clone loss} without \textit{attention loss} and MI-guided benign knowledge localization makes \ours{} inherit backdoor behaviors inevitably, as shown in the last two columns.
The observation is in line with our expectations.
The key to mitigate backdoors relies on BKL and \textit{attention loss}($L_2$).
Component BKL uses mutual information to locate benign knowledge, without which malicious knowledge is also transferred, leading to an uptrend in ASR.
Component $L_2$, \textit{attention loss} aligns the neurons on benign features instead of poisoned ones. This alignment strategy is instrumental in diminishing the impact of poisoned inputs.
\textit{Clone loss}($L_0$ and $L_1$) makes contributions to ACC from the second and third columns, without which the blue line representing the accuracy trend shows a decline
Respectively, $L_0$ minimizes the cosine similarity between the features extracted by teacher and student nets, through which the behavior of the student network is closer to the teacher network.
$L_1$ is achieved by contrastive loss as declared in Eq~\ref{eq:contrastive_loss}, following previous works~\cite{simclr}.
\ours{} uses data argumentation to construct positive pairs and different images as negative samples.
$L_1$ helps $L_0$ to focus on the image semantics to improve encoders' performance.
In summary, these components interact to ensure \ours{}'s effectiveness. 
\begin{figure}[tbp]
    \centering 
    \includegraphics[width=0.8\columnwidth]{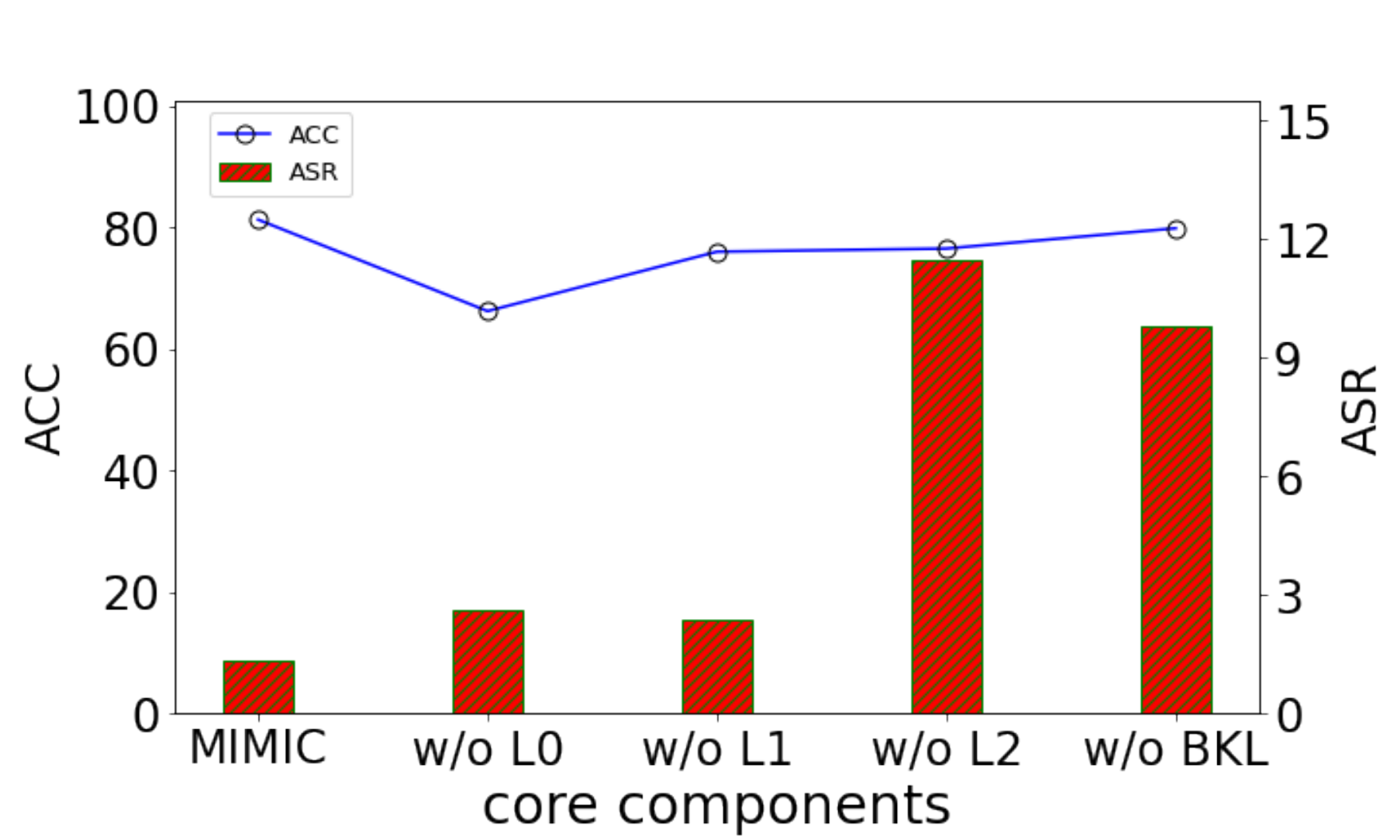}
    \caption{
    \revise{Influence of core components. The x-axis represents different configurations of the framework, including \ours{} and versions without specific components: $L_0$, $L_1$, $L_2$, and BKL (Benign Knowledge Localization module guided by MI). The left y-axis (blue line) shows the classification accuracy (ACC), while the right y-axis (red bar) shows the attack success rate (ASR). This comparison highlights the impact of removing each component on the framework's performance.}
    }
    \vspace{-12pt}
    \label{fig:ablation_core_components}
\end{figure}

\subsubsection{RQ2.2. How do the hype-parameters $\lambda_1$ and $\lambda_2$ affect \ours{}}
\begin{figure*}[htbp]
    \centering
    \subfloat[Hype-parameter: $\lambda_1$]{%
        \includegraphics[width=0.45\linewidth]{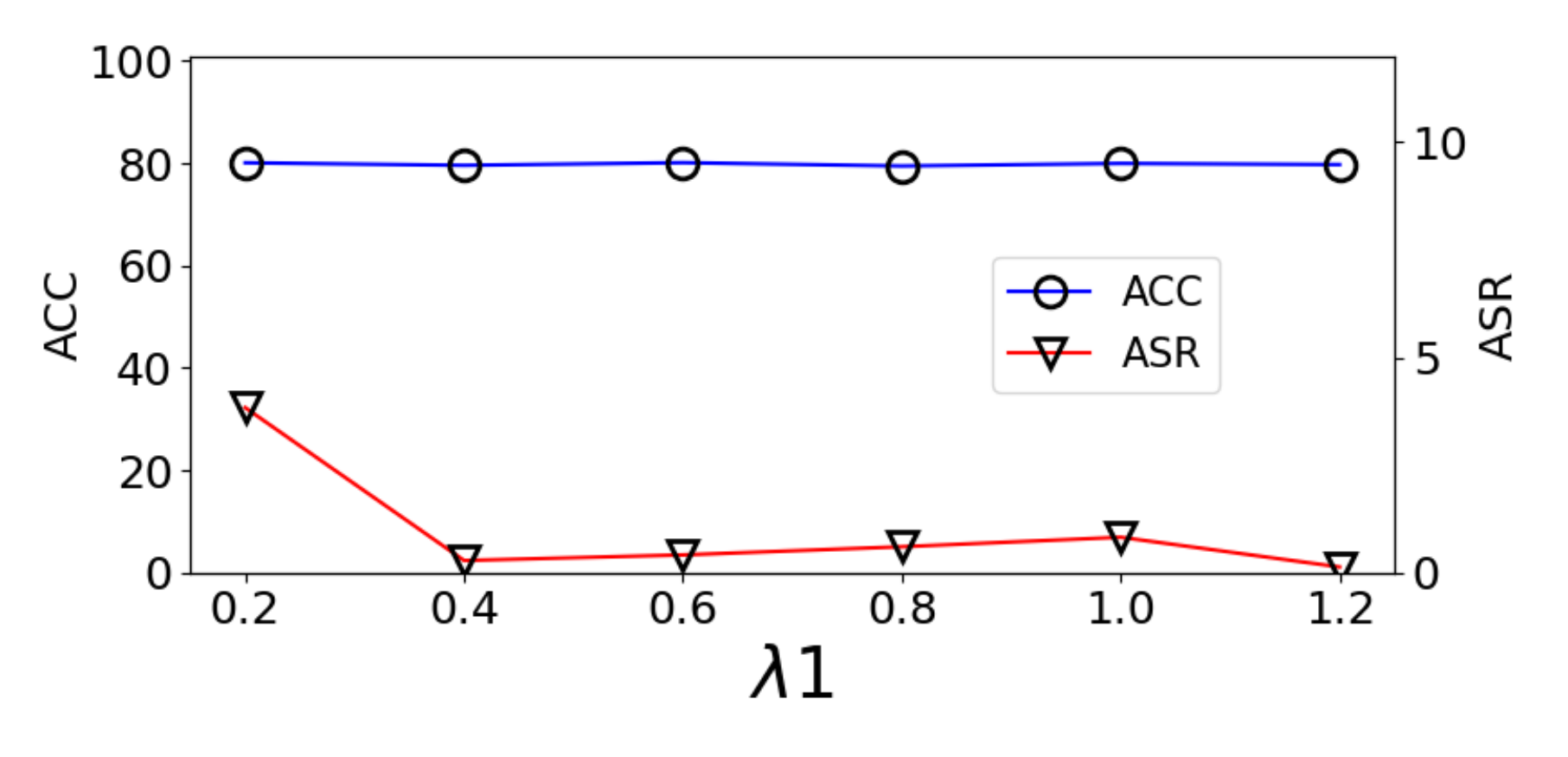}
        \label{fig:lambda1}
    }\hfill
    \subfloat[Hype-parameter: $\lambda_2$]{%
        \includegraphics[width=0.45\linewidth]{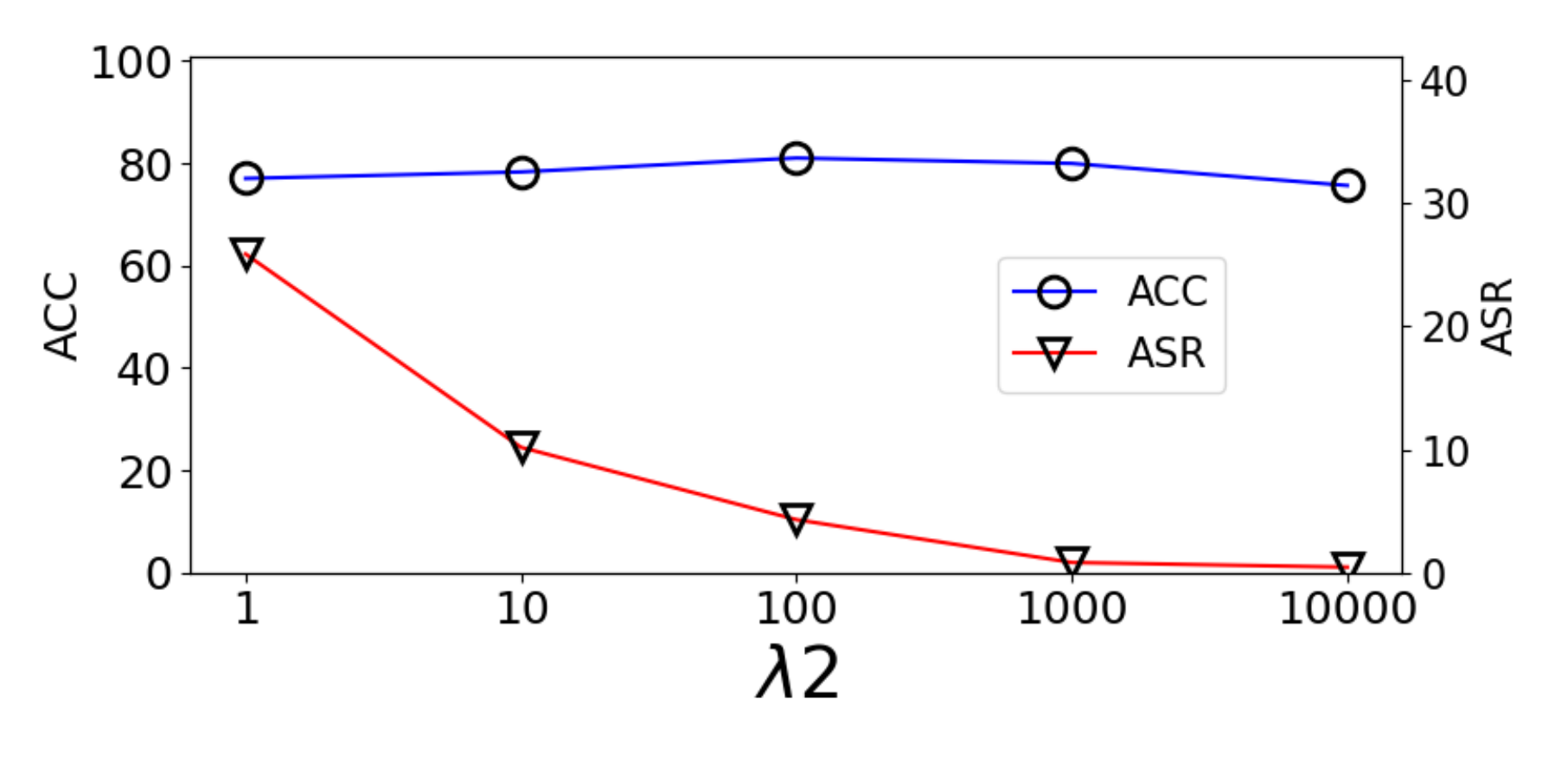}
        \label{fig:lambda2}
    }
     \caption{The influence of hype-parameters $\lambda_1$ and $\lambda_2$ on \ours{}. \revise{The x-axis represents specific hyper-parameter values (\(\lambda_1\) and \(\lambda_2\)). The left y-axis corresponds to accuracy (ACC), while the right y-axis represents the attack success rate (ASR).}} 
    \label{fig:lambda1_lambda2}
    \vspace{-12pt}
\end{figure*}

Figure~\ref{fig:lambda1_lambda2} presents the impact of $\lambda_1$ and $\lambda_2$ on MIMIC's performance.
The value ranges of $\lambda_1$ and $\lambda_2$ are determined for balancing the magnitude of different loss terms.
Observe that \ours{}'s ACC is stable.
The reason behind this is the existence of $L_0$ in Eq~\ref{eq:opt_loss}, detailed in Eq~\ref{eq:L_0}, aligning the benign features between teacher nets and student nets.
Moreover, the larger $\lambda_1$ and $\lambda_2$ are, the less ASR MIMIC achieves.
Particularly, $\lambda_1$ and $\lambda_2$ help MIMIC to focus more on benign knowledge of teacher nets.
When $\lambda_1>$0.4 and $\lambda_2>$1000, the ASR can be effectively reduced by MIMIC.

\subsection{RQ3. Evaluation Results on Robustness.}
\label{rq3:robustness}
\subsubsection{RQ3.1. How does the trigger size affect \ours{}}
In this section, we study how the size of the injected trigger impacts the performance of \ours{}.
We conduct experiments on three different sizes, including $3 \times 3$, $5 \times 5$, and $10 \times 10$. The experimental results are shown in Figure~\ref{fig:trigger_size}, where the red bar indicates the extent of ASR reduction after \ours{}, whereas the blue bar represents the corresponding ACC loss. 
Observed that \ours{} effectively removes backdoors of different sizes with acceptable degradation.

\subsubsection{RQ3.2. How does the clean data ratio affect \ours{}}
\revise{Mutual Information (MI) is a fundamental metric for quantifying the dependency between variables, encompassing both linear and nonlinear correlations. While it is well-suited for tasks requiring robust measures of relevance, MI is known to be a biased estimator, particularly when sample sizes are limited. This bias, however, decreases as the sample size increases, making it more reliable with larger datasets.
In our study, we explore the impact of this bias by varying the sample sizes of the clean data used for \ours{}.
The specific hyperparameter is clean data ratio ($\gamma$).
It represents the amount of clean data accessible to the defender, expressed as a ratio of the pre-training data.
}
The clean data ratio ($\gamma$), as a critical hyperparameter, consistently influences the performance of all defense methods. 
It represents the amount of clean data accessible to the defender, expressed as a ratio of the pre-training data.
We explore the influence of $\gamma$ by setting different values, taking five values ranging from 0.01 to 0.05 in steps of 0.01.
\revise{
Requiring $\leq 5$\% clean data in the assumed defense scenario is a realistic and commonly adopted setting in current backdoor mitigation techniques~\cite{2021-NAD,2021-ANP,2022-MOTH}. This small ratio aligns with real-world practices where organizations using pre-trained encoders from third-party sources often maintain access to trusted datasets representative of their operational domain. For instance, in critical sectors like healthcare, finance, and autonomous systems, organizations typically retain a subset of verified clean data for tasks such as validation or fine-tuning. This practice ensures that models are reliable and perform well within specific application contexts.
}
We use CIFAR10 as the pre-training dataset and STL10 as the downstream task dataset for experiments. 
The experimental results are presented in Figure~\ref{fig:data_ratio}, where the x-axis represents the clean data ratio, the left y-axis represents ACC, and the right-axis represents ASR.
It is observed that as the ratio of clean data gradually increases, the ACC of \ours{} increases gradually, while the ASR decreases consistently and reaches the best performance at 0.05.
This observation is intuitive as a larger $\gamma$ corresponds to defenders having access to a greater amount of extra knowledge, thereby enabling defensive approaches to achieve more remarkable performance outcomes.
\revise{Empirical results demonstrate that, even with no more than 5\% clean data, MI remains sufficiently effective in guiding the identification of benign knowledge within backdoored encoders. As the sample size increases, the consistency and precision of the MI estimates improve, further enhancing the robustness of our method.}

\subsubsection{RQ3.3. How does the poison ratio affect \ours{}}
In backdoor removal, it is a common threat model assuming a confident small set of clean training data,
which aligns with the literature~\cite{2022-MOTH, 2021-NAD, 2021-ANP}.
We evaluate the performance of MIMIC when the given clean set contains poisoned data.
As shown in the upper right figure, the x-axis denotes different poison ratios on the defender-retained data.
MIMIC can consistently reduce ASR, even when 25\% of the set is poisoned, delineating the robustness of MIMIC.
The results indicate that a small amount of malicious data cannot break MIMIC's effective mechanism.

\subsubsection{RQ3.4. What is the performance of \ours{} against adaptive attack}
\ours{} removes backdoors by cloning benign knowledge from backdoored encoders to empty ones with the guidance of mutual information.
During evaluation, we conduct an adaptive attack that aims to evade our defense.
The attack adds a regularization loss to minimize the distance between benign and poisoned knowledge measured by mutual information.
Figure~\ref{fig:adaptive_attack} reports the results.
Observe that the ASR of adaptive attack is only 33.29\%. This is because the regularization loss forces the benign and poisoned features to be the same, which contradicts the main attack objective.
For this adaptively attacked encoder, MIMIC can still reduce the ASR to 11.88\%.

\begin{figure}[t] 
    \centering
    \begin{minipage}[t]{0.49\columnwidth}
        \centering
        \includegraphics[width=\columnwidth]{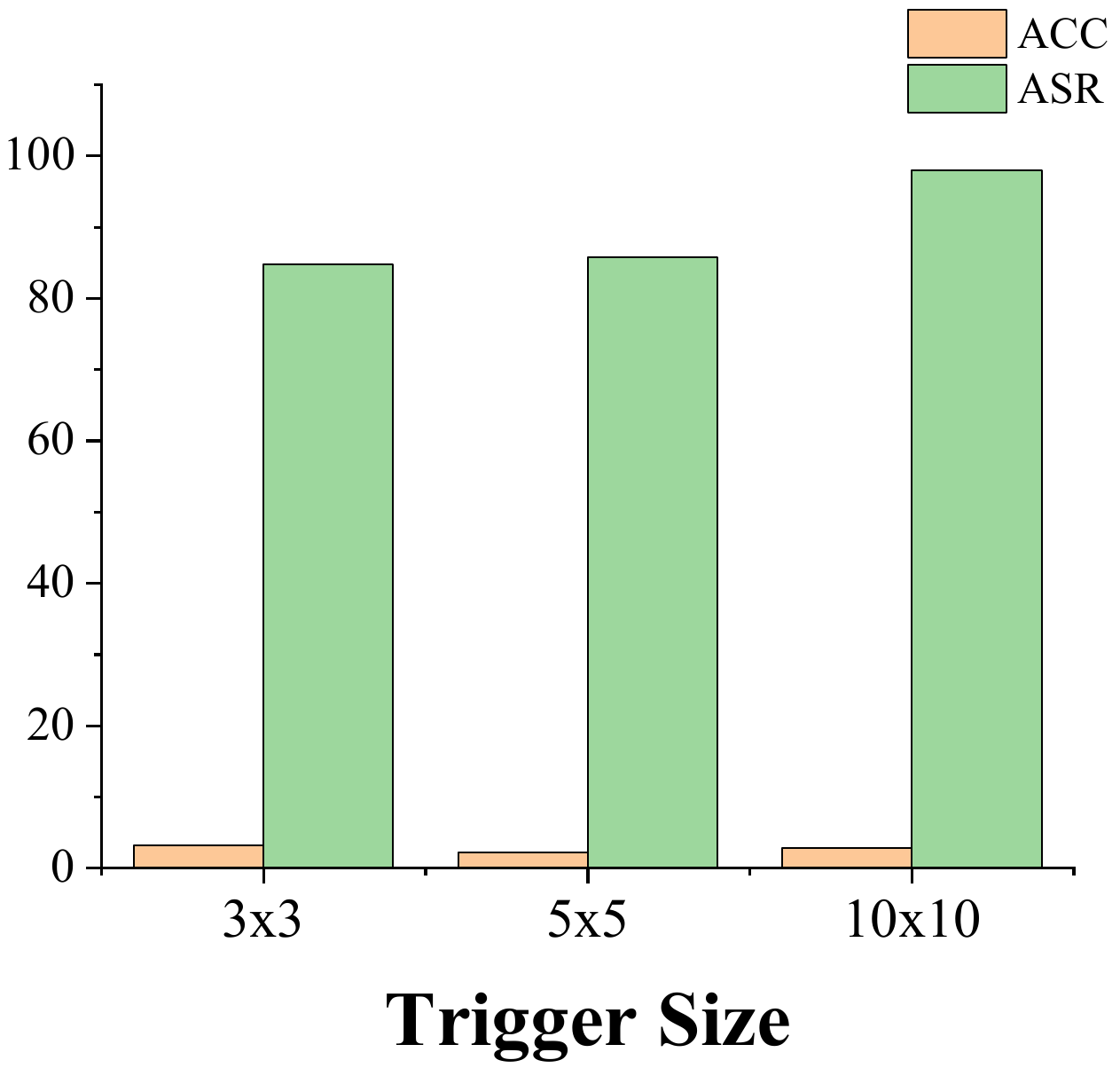}
        \caption{Effect of trigger size}
        \label{fig:trigger_size}
    \end{minipage}
    \hfill 
    \begin{minipage}[t]{0.49\columnwidth} 
        \centering
        \includegraphics[width=\columnwidth
        ]{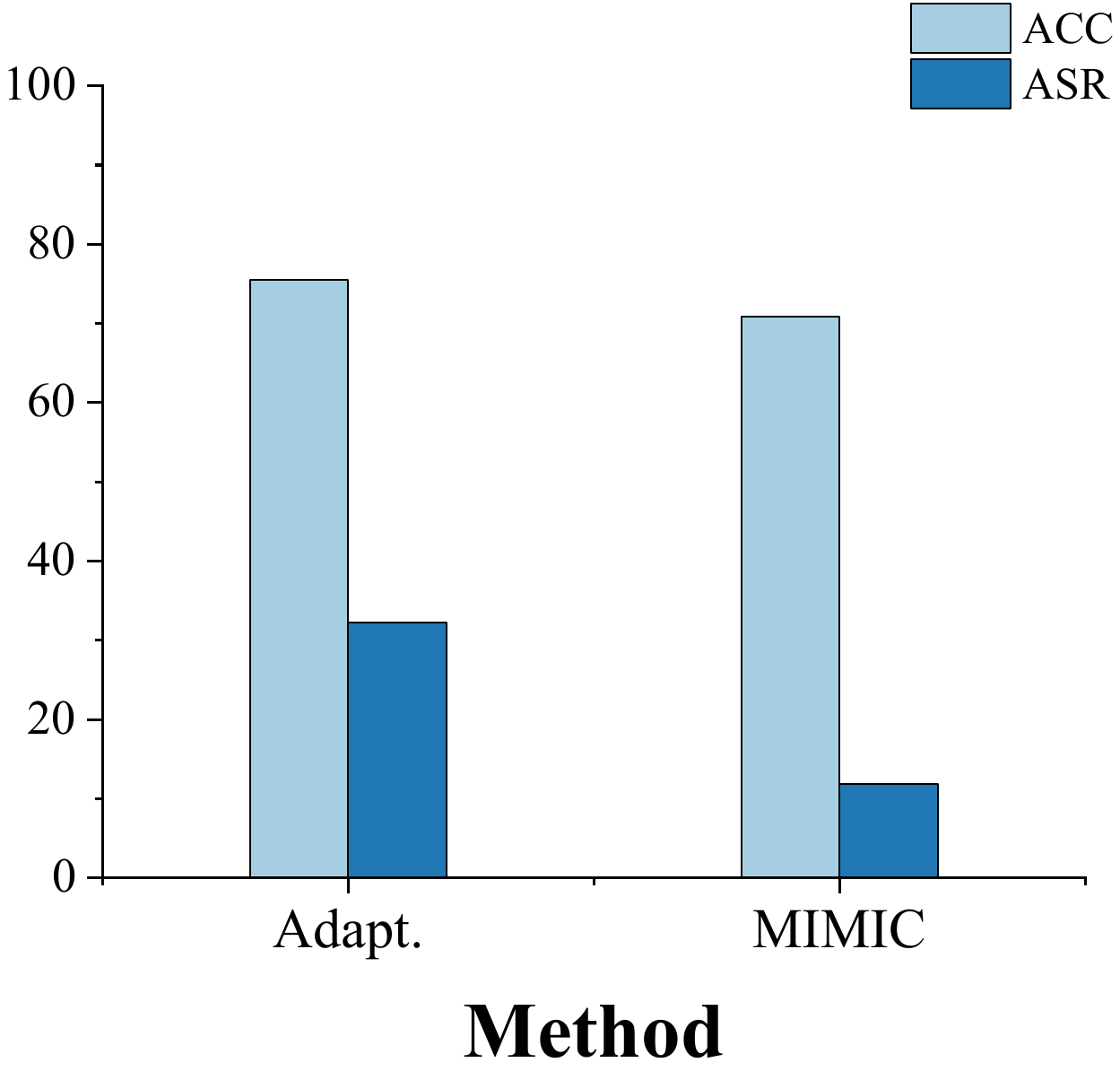}
        \caption{Adaptive attack}
        \label{fig:adaptive_attack}
    \end{minipage}
\end{figure}
\vspace{-12pt}

\begin{figure}[t] 
    \centering
    \begin{minipage}[t]{0.49\columnwidth}
        \centering
        \includegraphics[width=\columnwidth]{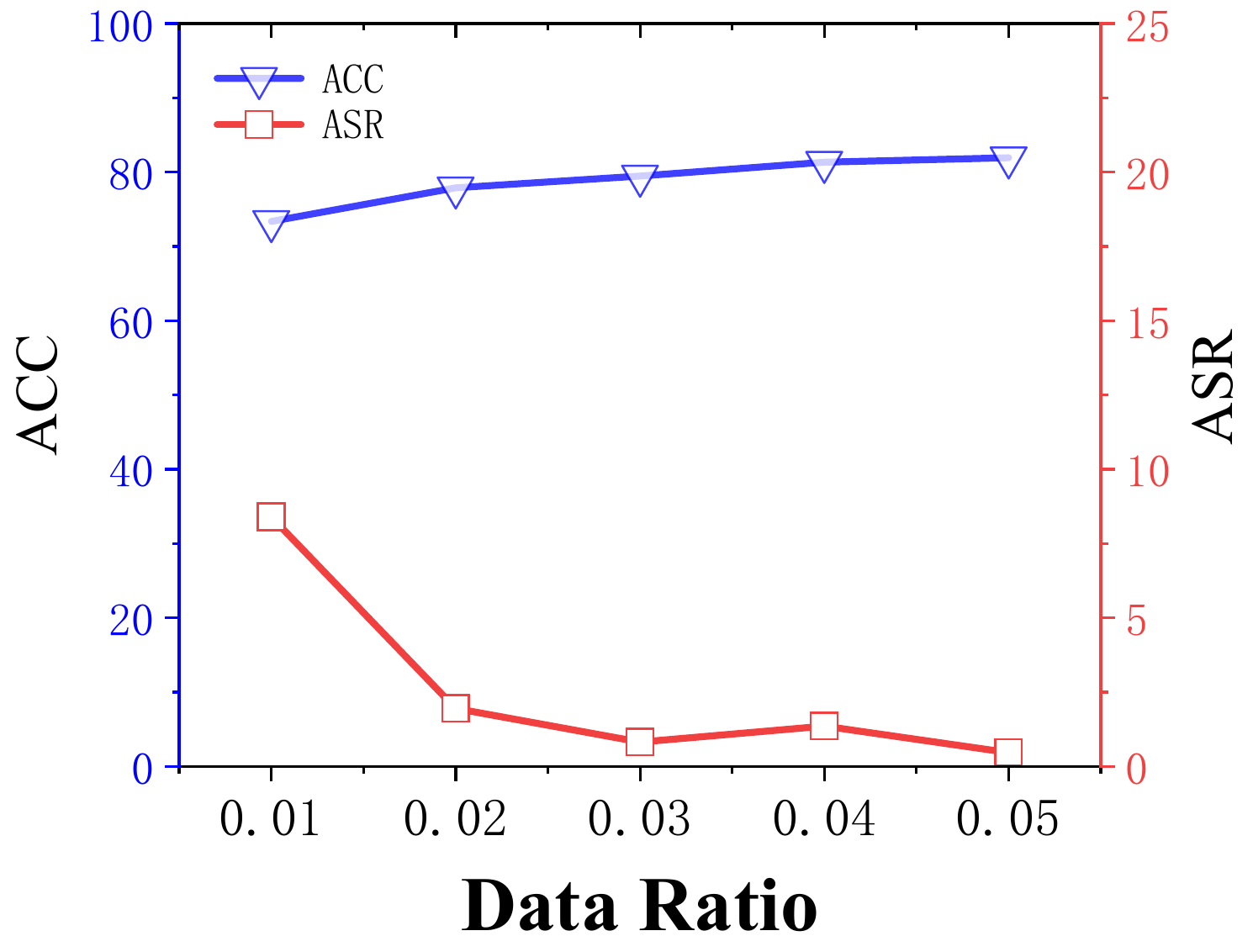}
    \caption{Effect of clean data}
    \label{fig:data_ratio}
    \end{minipage}
    \hfill 
    \begin{minipage}[t]{0.49\columnwidth} 
        \centering
        \includegraphics[width=\columnwidth
        ]{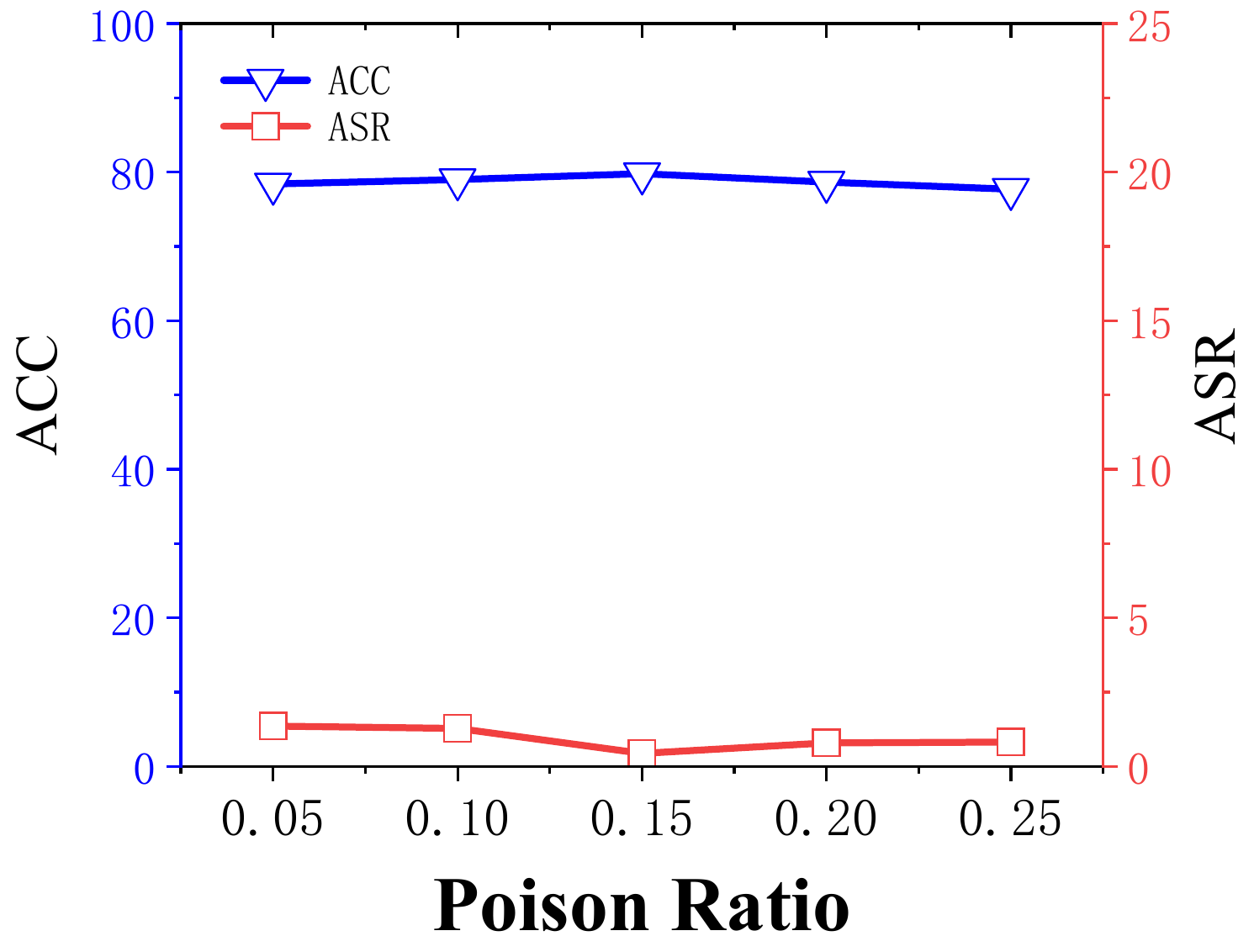}
    \caption{Effect of poison ratio}
    \label{fig:poison_ratio}
    \end{minipage}
\end{figure}

\subsection{RQ4. Evaluation Results on Generalization.}
\label{rq2:generalization}
\subsubsection{RQ4.1. Will MIMIC continue its performance when
extending to supervised learning}
\begin{table}[ht]
    \centering
    \footnotesize
    \caption{Extensive study on SL}
    \begin{tabular}{ccccccc}
        \toprule
        \multirow{2}{*}{Method} & \multicolumn{2}{c}{Four Corner} & \multicolumn{2}{c}{Grid Trigger} & \multicolumn{2}{c}{Random Pixel} \\ 
        \cmidrule(lr){2-3} \cmidrule(lr){4-5} \cmidrule(lr){6-7}  
         & ACC & ASR & ACC & ASR & ACC & ASR \\
         \cmidrule(lr){1-7}
        Undefended & 83.89 & 100 & 84.37 & 100 & 84.01 & 100 \\
        ~\ours{} & 78.13 & 5.77 & 78.01 & 2.50 & 77.67 & 0.01 \\
        \bottomrule
    \end{tabular}
    \label{tab:extensive_sl}
\end{table}

Given that classifiers also exhibit a linear structure suiting our key intuition, we expand \ours{} to encompass the realm of supervised learning. 
The defenders are available to 0.05\% of the training data and have no other knowledge of attack or model training schedule.
We conduct \ours{} on BadNets, a typical attack on supervised learning, with three different triggers~\cite{2017-BadNets}.
Table~\ref{tab:extensive_sl} exhibits the experimental results, where CIFAR10 and resnet18 are used as dataset and model architecture, respectively.
Note that \ours{} have the competence in reducing ASR to 5\%, albeit at the expense of ACC loss of approximately 7\%.
This phenomenon can be attributed to \ours{} employing an empty model as the student net, which impacts classifier performance.
The prospect of such extensions remains a subject of future investigation.
\subsubsection{RQ4.2. How does MIMIC react to clean encoders}
\begin{table}[htb]
    \centering
    \tabcolsep=13.5pt
    \caption{Impact on clean encoders}
    \begin{tabular}{cccc}
        \toprule
         \multirow{2}{*}{Pre-train} & \multirow{2}{*}{Downstream} & \multicolumn{2}{c}{Accuracy} \\ 
        
        \cmidrule{3-4} 
        & & Raw & \ours{} \\
         
        \cmidrule{1-4}
        \multirow{3}{*}{CIFAR10} & STL10 & 75.31 & 72.22 \\
        
        & GTSRB & 82.42 & 81.33 \\
        
        & SVHN & 57.01 & 75.32 \\
        \bottomrule
    \end{tabular}
    \label{tab:impact_clean_encoder}
\end{table}

Considering the fact that defenders lack knowledge where a given encoder is backdoored, we investigate the ramifications of \ours{} in scenarios where clean encoders are provided.
Table~\ref{tab:impact_clean_encoder} showcases the results of our experimental investigations involving the application of \ours{} on clean encoders.
\revise{Column ``raw'' denotes the performance metrics of downstream models when utilizing the original, unprocessed encoder directly, without any backdoor mitigation or purification methods applied.}   
It is noteworthy that \ours{} effectively manages to control the reduction in accuracy to a level that is considered entirely acceptable.
\revise{
Observe that \ours{} improves the accuracy for SVHN (75.32 vs. 57.01).  It reflects the compatibility of \ours{}'s distilled encoder with the characteristics of the SVHN dataset.  The digit classification task in SVHN benefits from the focused feature extraction provided by \ours{}, as it emphasizes core features that align well with the dataset’s relatively constrained feature space.  This differs from other downstream tasks, where the complexity of the features and the trade-off between backdoor mitigation and accuracy retention may result in slight performance reductions.
This observation is also consistent with previous works~\cite{2022-BadEncoder}.
}

\subsubsection{RQ4.2. How much time does \ours{} cost}
\revise{
The time efficiency of \ours{} is one of its strengths, particularly when compared to the second-best baseline NAD~\cite{2021-NAD}.
\ours{} first identifies the benign knowledge by training the mutual information estimator. 
This process takes approximately 680 seconds, which accounts for the estimation across layers with a manageable clean dataset.
With the guidance of mutual information weights, the student encoder is trained to distill benign knowledge from the poisoned teacher encoder. This process requires an additional 3,199 seconds.
\ours{} takes 3,879 seconds in total.
In contrast, NAD, a state-of-the-art baseline, requires significantly more computation, needing 6,248 seconds due to its reliance on repeated fine-tuning to address backdoor behavior inherited during distillation.
\ours{}, with its lightweight mutual information-guided approach, effectively reduces attack success rates while being far more time-efficient, making it a practical and scalable solution for backdoor mitigation.
}

\section{Conclusion}
\label{sec:conclusion}
In this work, we propose \ours{}, a novel framework leveraging mutual information to mitigate backdoor attacks in pre-trained encoders within self-supervised learning (SSL). Beyond achieving significant reductions in attack success rates and maintaining competitive accuracy, our work offers key insights into the unique challenges of backdoor mitigation in SSL, such as the absence of labeled data and task-specific context.
Our findings underscore the critical role of mutual information in identifying and preserving benign knowledge while suppressing malicious behaviors, paving the way for more effective and adaptable defense mechanisms. While \ours{} demonstrates robustness across diverse datasets and attack scenarios, we acknowledge certain trade-offs, such as slight accuracy reductions for complex downstream tasks, and identify these as areas for future optimization.
Our study highlights the importance of further exploring data-efficient defense strategies and extending \ours{} to address adaptive attacks and more sophisticated threat models. By bridging the gap between security and performance in SSL, we aim to inspire broader research into secure and trustworthy AI.

\appendix
\section{Appendix}
\label{sec:appendix}

\noindent\textit{Proof sketch of Theorem~\ref{theo:markov_benign_features}.}
We first prove that $Z\rightarrow \mathcal{F}_\theta^{n-1}\rightarrow \mathcal{F}_\theta^{n-2}\rightarrow\cdots\mathcal{F}_\theta^{0}\rightarrow X$ is a Markov chain. 
Then we deploy data processing inequality to achieve the final conclusion.
\begin{assumption}[Markov chain]
    Considering a neural net comprising of N layers, $X \rightarrow \mathcal{F}_\theta^0 \rightarrow \cdots   \rightarrow \mathcal{F}_\theta^{N-1} \rightarrow Z$ is a Markov chain.
    \label{assump:markov}
\end{assumption}
\begin{proposition}[Reversed Markov chain]
    If $X \rightarrow \mathcal{F}_\theta^0 \rightarrow \cdots \rightarrow \mathcal{F}_\theta^{N-1} \rightarrow Z$ is a Markov chain, then $Z \rightarrow \cdots \rightarrow \mathcal{F}_\theta^{N-1} \rightarrow X$ is also a Markov chain.
    \label{prop:re_markov}
\end{proposition}
\begin{proof} 
We utilize mathematical induction to prove it.

\textbf{Base case.} When $N=1$, we need to prove $Z\rightarrow\mathcal{F}_\theta^0\rightarrow X$ is a Markov chain.
Given that $X\rightarrow\mathcal{F}_\theta^{n-1}\rightarrow Z$ is a Markov chain, then $p(z|x,h_{0})=p(z|h_{0})$, where $z \in Z$, $p(\cdot)$ the probability distribution and $h_l$ the output of layer $\mathcal{F}_\theta^l$.
\begin{align*}
p(x|h_0, z)&=\frac{p(x,h_0,z)}{p(h_0,z)}\\
&=\frac{p(z|x,h_0)\cdot p(x,h_0)}{p(h_0,z)}\\
&=\frac{p(z|x,h_0) \cdot p(x|h_0)\cdot p(h_0)}{p(z|h_0)\cdot p(h_0)}\\
&=p(x|h_0)
\end{align*}
Hence, the base case holds.

\textbf{Inductive Hypothesis.}
Assume that for some positive integer k, the equation holds:
$Z\rightarrow \mathcal{F}_\theta^{k-1}\rightarrow \mathcal{F}_\theta^{k-2}\rightarrow\cdots\mathcal{F}_\theta^{0}\rightarrow X$ is a Markov chain.

\textbf{Inductive Step.}
We need to prove that it holds for $k+1$, $Z\rightarrow \mathcal{F}_\theta^{k}\rightarrow \cdots \mathcal{F}_\theta^{1}\rightarrow \mathcal{F}_\theta^{0}\rightarrow X$ is a Markov chain, i.e., to prove:
$p(x|z,h_k,\cdots,h_0)=p(x|h_0)$.
 
Based on Assumption~\ref{assump:markov}, we have $X\rightarrow\mathcal{F}_\theta^{0}\rightarrow\cdots\mathcal{F}_\theta^{k-1}\rightarrow\mathcal{F}_\theta^{k}\rightarrow Z$ is a Markov chain, then  $p(z|x, h_{k},h_{k-1},\cdots,h_0)=p(z|h_k)$.
Hence,
\begin{align*}
    &p(x|z,h_k,h_{k-1},\cdots,h_0)\\
    =&\frac{p(z,x,h_k,h_{k-1},\cdots,h_0)}{p(z,h_{k},h_{k-1},\cdots,h_0)}\\
    =&\frac{p(z|x,h_k,h_{k-1},\cdots,h_0)\cdot p(x,h_k,h_{k-1},\cdots,h_0)}{p(z|h_k,h_{k-1},\cdots,h_0) \cdot p(h_k,h_{k-1},\cdots,h_0)}\\
    =&\frac{p(z|h_k)}{p(z|h_k)}\cdot\frac{p(x,h_k,h_{k-1},\cdots,h_0)} {p(h_k,h_{k-1},\cdots,h_0)}\\
    =&\frac{p(h_k|x,h_{k-1},\cdots,h_0)\cdot p(x,h_{k-1},\cdots,h_0)}{p(h_k|h_{k-1},\cdots,h_0)\cdot p(h_{k-1},\cdots,h_0)}\\
    =&\frac{p(x,h_{k-1},\cdots,h_0)}{p(h_{k-1},\cdots,h_0)}\\
    =&\cdots \\
    =&\frac{p(x,h_0)}{p(h_0)}\\
    =&p(x|h_0)
\end{align*}
\end{proof}

\begin{theorem}[Data Processing Inequality]
     If $X\rightarrow\mathcal{F}_\theta^{0}\rightarrow\mathcal{F}_\theta^{1}\rightarrow\cdots\mathcal{F}_\theta^{n-1}\rightarrow Z$ is a Markov chain, then $I(\mathcal{F}_\theta^{n-1},X)\leq I(\mathcal{F}_\theta^{n-2},X)\leq\cdots\leq I(\mathcal{F}_\theta^{1},X)\leq I(\mathcal{F}_\theta^{0},X)$.
     \label{theo:data_pro_ineq}
\end{theorem}

Combined with Proposition~\ref{prop:re_markov} and Theorem~\ref{theo:data_pro_ineq}~\coloredcite{beaudry2011intuitive}, we have
$I(\mathcal{F}_\theta^{0},Z)\leq I(\mathcal{F}_\theta^{1},Z)\leq\cdots\leq I(\mathcal{F}_\theta^{n-2},Z)\leq I(\mathcal{F}_\theta^{n-1},Z)$.
Then Theorem~\ref{theo:markov_benign_features} is achieved.

\begin{figure*}[!thbp]
    \centering
    \subfloat[Original image]{%
        \includegraphics[width=0.22\linewidth]{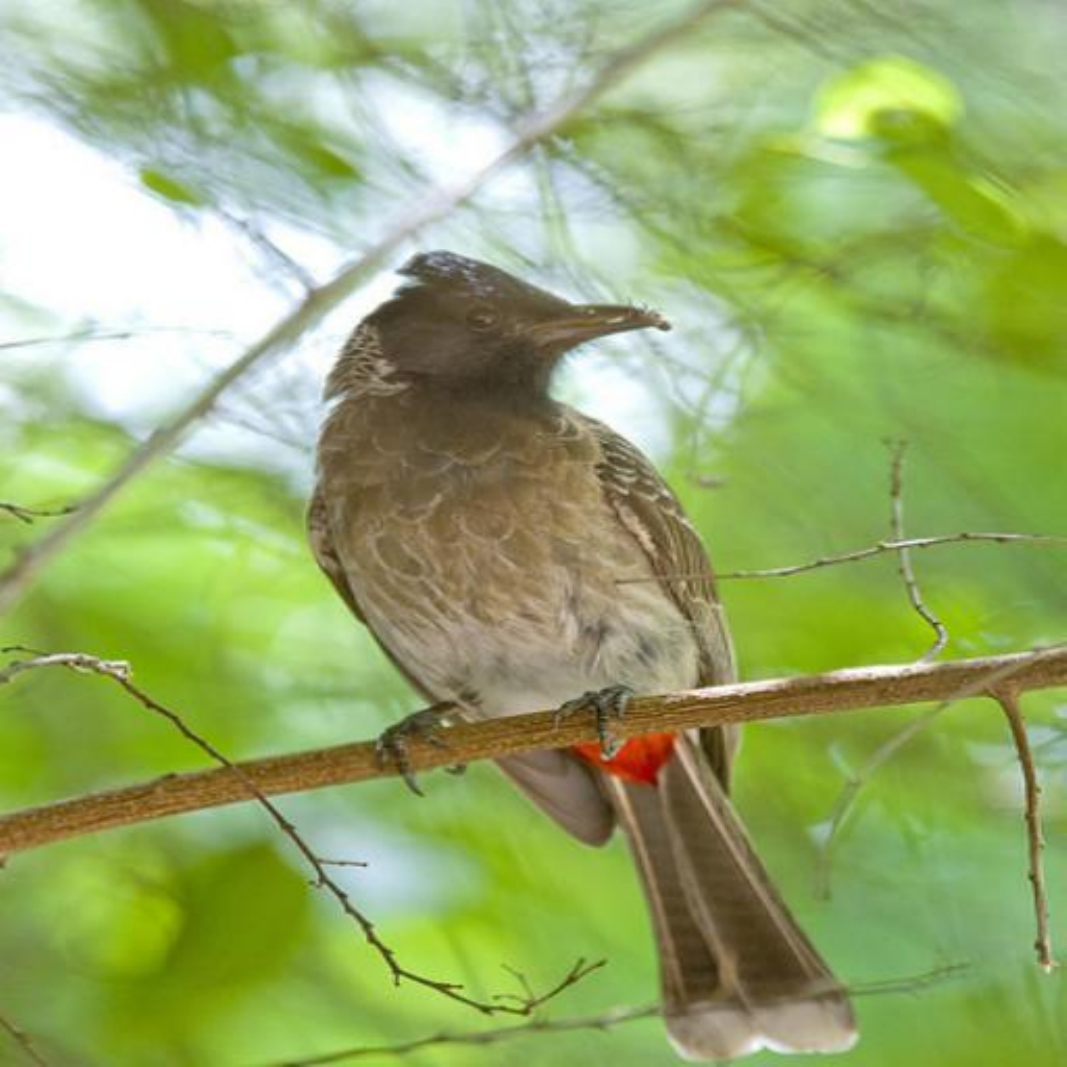}
        \label{fig:original_image}
    }\hfill
    \subfloat[Horizontal flipping]{%
        \includegraphics[width=0.22\linewidth]{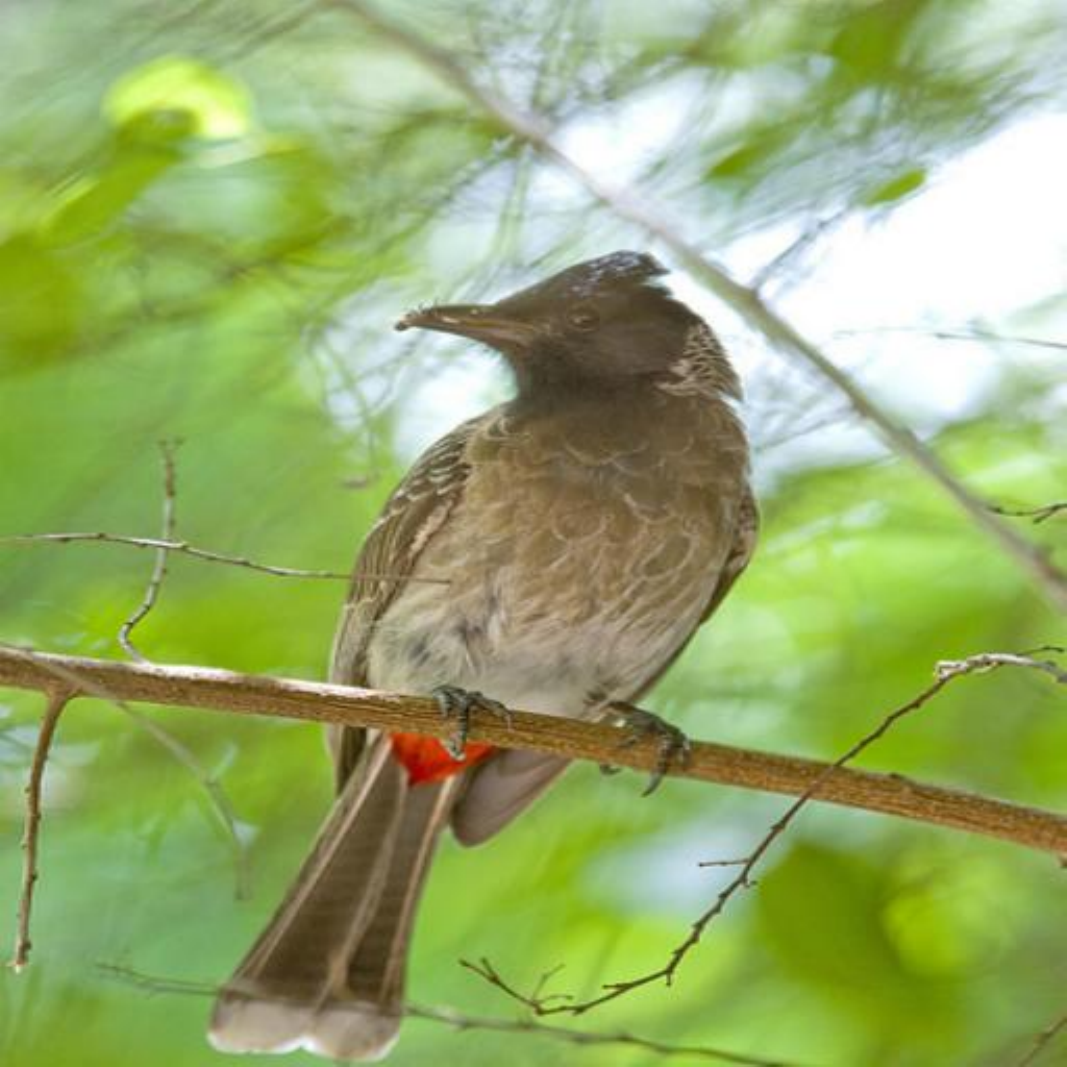}
        \label{fig:flip}
    }\hfill
    \subfloat[Color jittering]{%
        \includegraphics[width=0.22\linewidth]{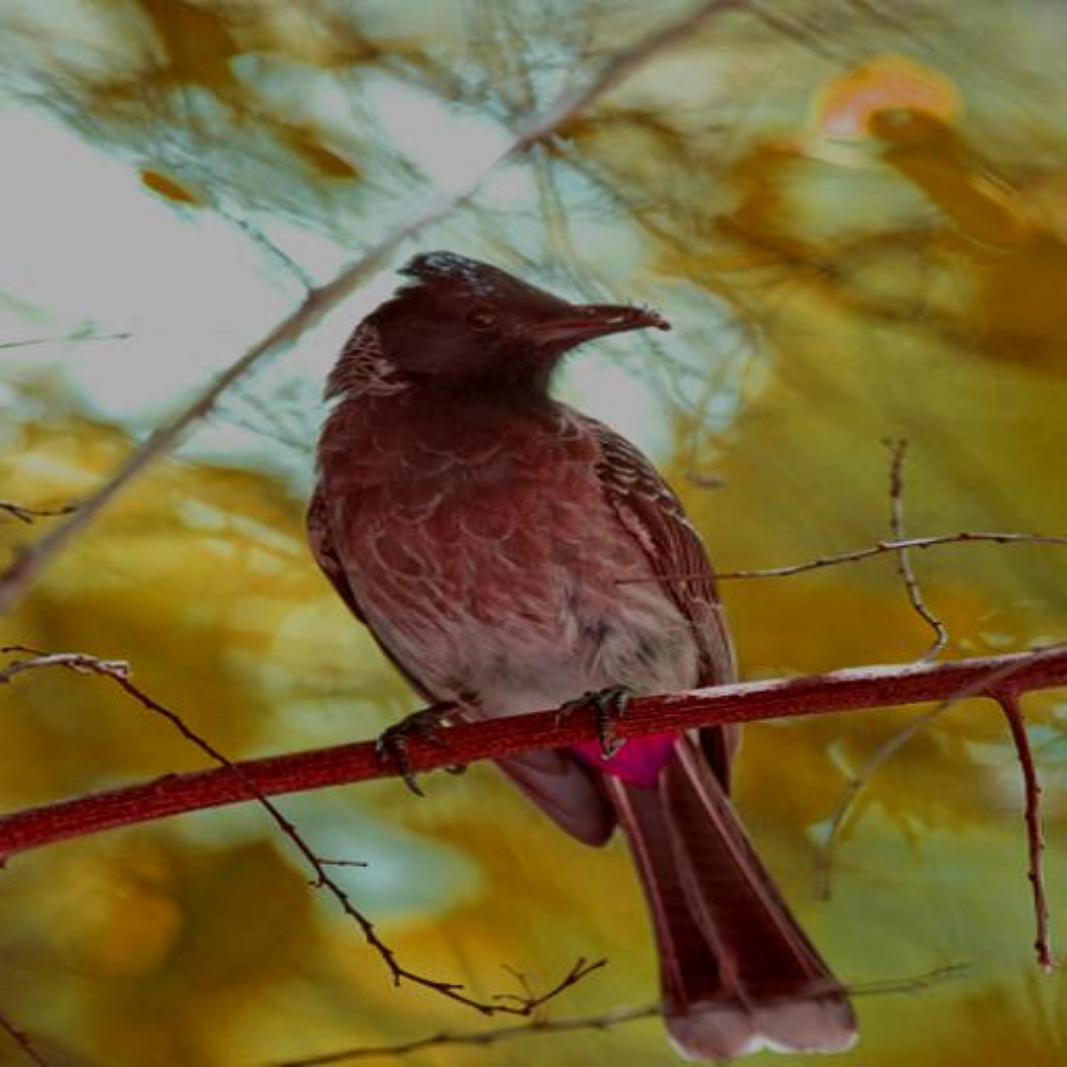}
        \label{fig:jitter}
    }\hfill
    \subfloat[Grayscale conversion]{%
        \includegraphics[width=0.22\linewidth]{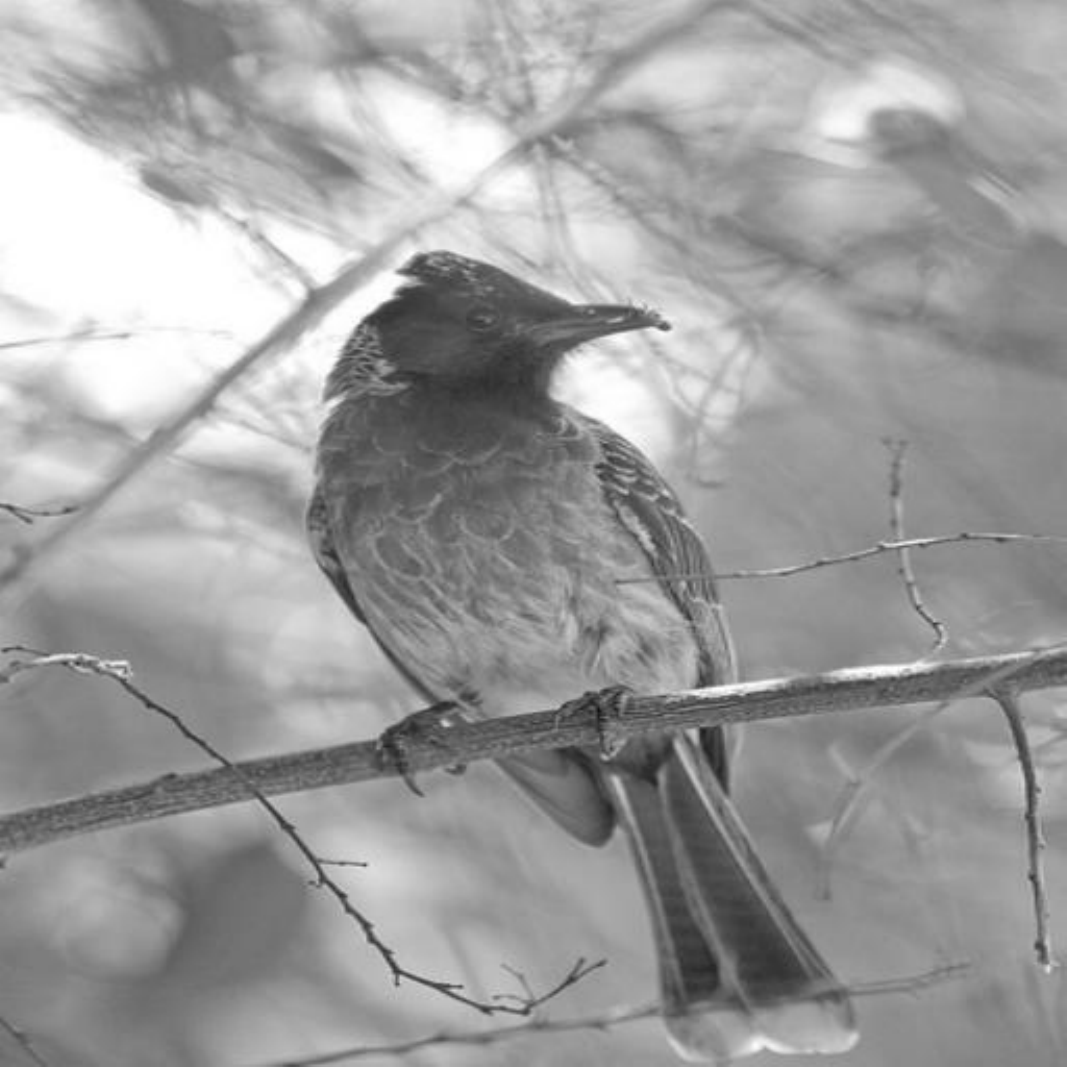}
        \label{fig:gray}
    }
     \caption{\revise{Illustration of the augmentation methods used in our study, including horizontal flipping (Figure~\ref{fig:flip}), color jittering (Figure~\ref{fig:jitter}), and grayscale conversion (Figure~\ref{fig:gray}).}}
    \label{fig:diverse_augmentation}
    \vspace{-12pt}
\end{figure*}

{
    \small
    \bibliographystyle{IEEEtran}
    \bibliography{reference}
}

\begin{IEEEbiography}[{\includegraphics[width=1in,height=1.25in,clip,keepaspectratio]{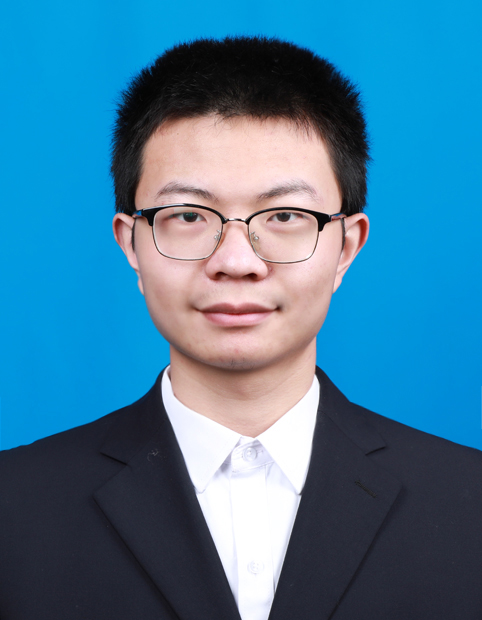}}]{Tingxu Han}
is currently working toward the Ph.D. degree in Software Institute at Nanjing University, Nanjing, China.
His research interest includes AI security and adversarial training.
\end{IEEEbiography}

\begin{IEEEbiography}[{\includegraphics[width=1in,height=1.25in,clip,keepaspectratio]{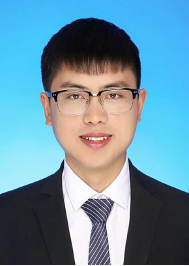}}]{Weisong Sun} is currently a research fellow at the College of Computing and Data Science, Nanyang Technological University, Singapore. He received a Ph.D. degree in Software Engineering from Nanjing University, China in 2023. His research interests include intelligent software engineering, trustworthy artificial intelligence (especially AI model security), and research spanning both fields (e.g., trustworthy intelligent software engineering). He has more than 35 high-quality publications including TIFS, TDSC, TSE, TOSEM, ICSE, ESEC/FSE, ASE, ISSTA, ACL, etc. He served as the reviewer of ICSE, ASE, TSE, TOSEM, ACL, TR, IJHC, QRS, etc. In addition, he served as the co-chair of the International Workshop on AI Reliability and Security (AIRS 2024).
\end{IEEEbiography}

\begin{IEEEbiography}
[{\includegraphics[width=1in,height=1.25in,clip,keepaspectratio]{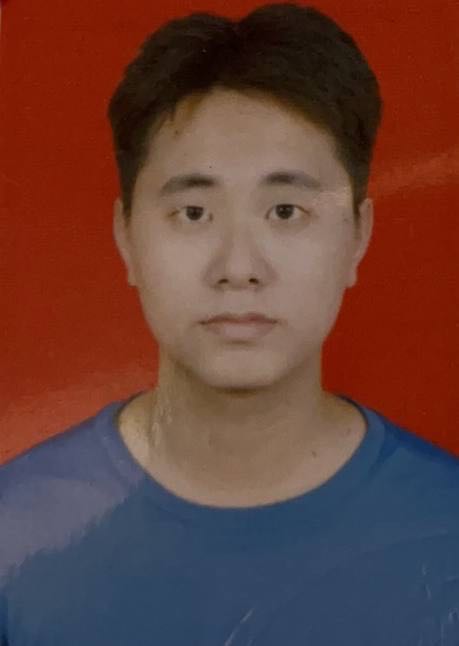}}]
{Ziqi Ding}
is studying for a master’s degree in the School of Computer Science and Engineering, University of New South Wales. His research interests involve AI Security.
\end{IEEEbiography}

\begin{IEEEbiography}[{\includegraphics[width=1in,height=1.25in,clip,keepaspectratio]{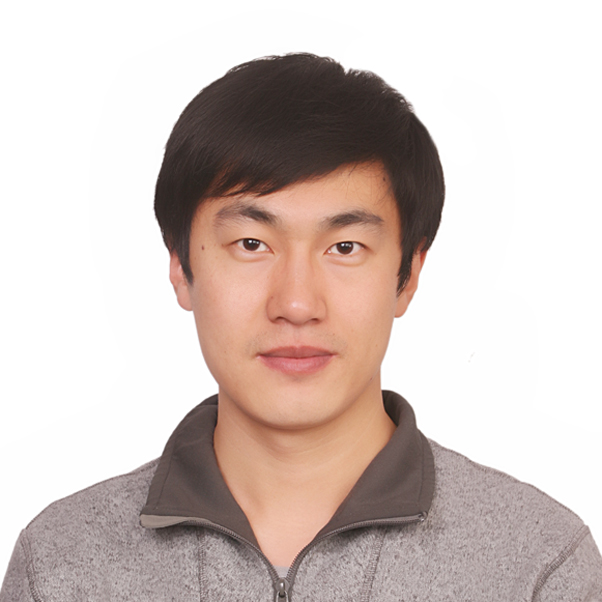}}]{Chunrong Fang}
received the B.E. and Ph.D. degrees in software engineering from Software Institute, Nanjing University, Jiangsu, China. He is currently an assistant professor with the Software Institute of Nanjing University. His research interests lie in intelligent software engineering, e.g. BigCode and AITesting.
\end{IEEEbiography}

\begin{IEEEbiography}[{\includegraphics[width=1in,height=1.25in,clip,keepaspectratio]{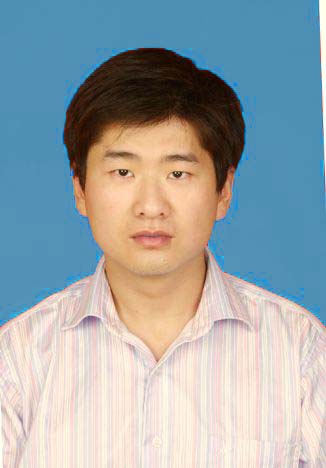}}]{Hanwei Qian}
is a Ph.D. candidate in the Software Institute at Nanjing University, Nanjing, China. His research interests lie in intelligent software engineering and the security of artificial intelligence (AI) models.
\end{IEEEbiography}

\begin{IEEEbiography}[{\includegraphics[width=1in,height=1.25in,clip,keepaspectratio]{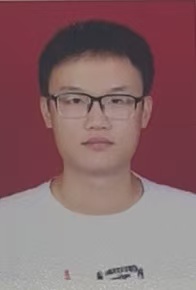}}]{Jiaxun Li}
is studying for a master's degree in the School of Mathematical Sciences at Soochow University. His research direction is Lie groups and Lie algebra
\end{IEEEbiography}

\begin{IEEEbiography}[{\includegraphics[width=1in,height=1.25in,clip,keepaspectratio]{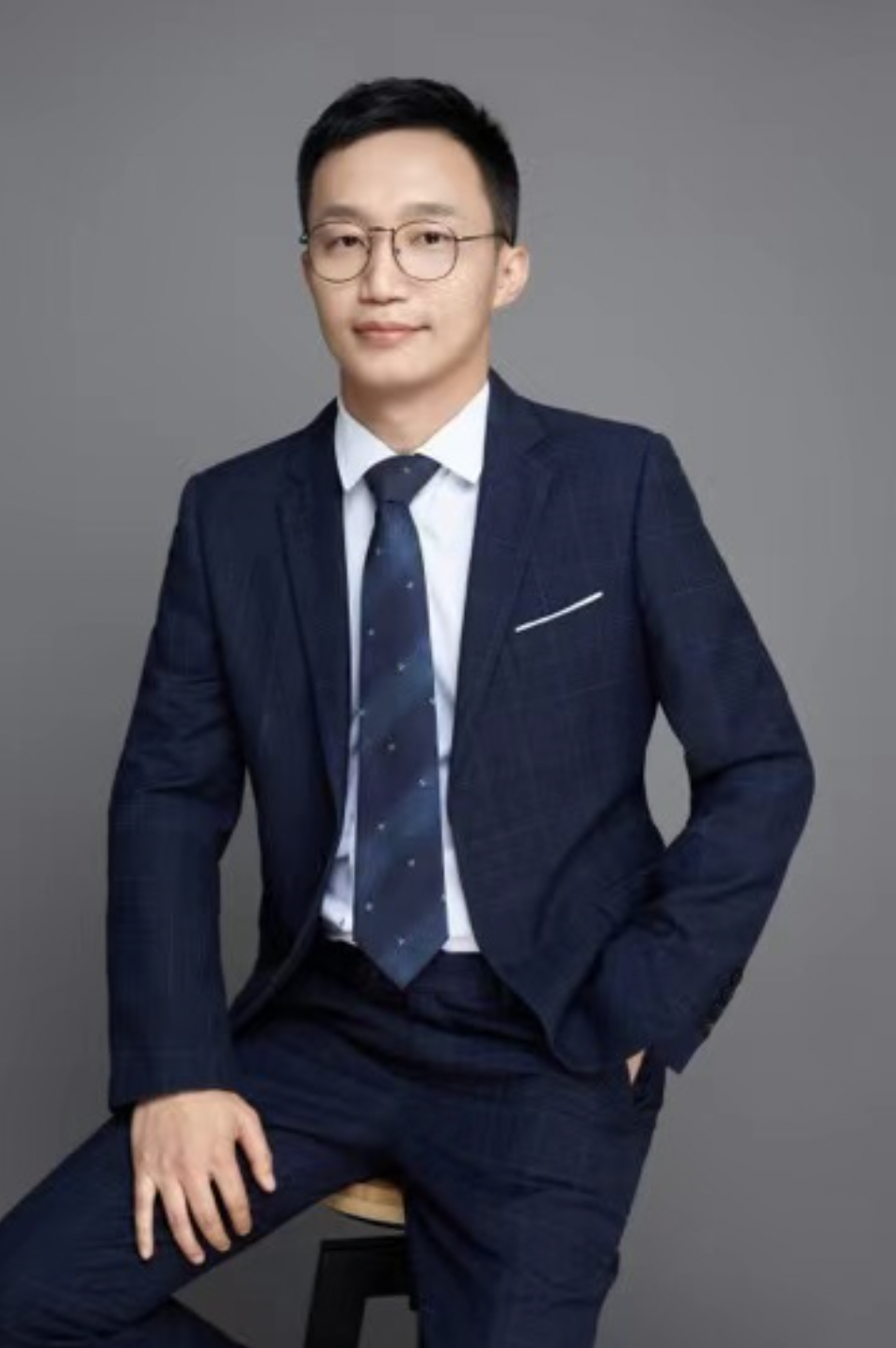}}]{Zhenyu Chen}
is currently a full professor with Software Institute of Nanjing University. He is an associate Editor of IEEE Transactions on Reliability. He is also the Contest Co-Chair at QRS 2018, ICST 2019, and ISSTA 2019. He is the Industrial Track Co-Chair of SANER 2019. His research interests include collective intelligence, deep learning testing and optimization, big data quality, and mobile application testing.
\end{IEEEbiography}

\begin{IEEEbiography}[{\includegraphics[width=1in,height=1.25in,clip,keepaspectratio]{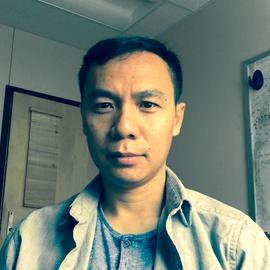}}]{Xiangyu Zhang}
is a professor specializing in AI security, software analysis and cyber forensics. His work involves developing techniques to detect bugs, including security vulnerabilities, in traditional software systems as well as AI models and systems, and to diagnose runtime failures. He has served as the Principal Investigator (PI) for numerous projects funded by organizations such as DARPA, IARPA, ONR, NSF, AirForce, and industry. Many of the techniques developed by his team have successfully transitioned into practical applications. His research outcome has been published on top venues in the areas of Security, AI, Software Engineering, and Programming Languages and recognized by various distinguished paper awards, including the prestigious ACM Distinguished Dissertation Awards. He has mentored over 30 PhD students and post-docs, with fifteen of them securing academic positions in various universities. Many of them have been honored with NSF Career Awards or comparable recognitions.
\end{IEEEbiography}

\vfill

\end{document}